\documentclass[10pt]{article} 
\usepackage[preprint]{rlc}
\usepackage{amssymb}            
\usepackage{amsthm}
\usepackage{bbm}

\usepackage{mathtools}          
\usepackage{mathrsfs}           
\usepackage{graphicx}           
\usepackage{subcaption}  
\usepackage{xcolor}
\usepackage[space]{grffile}     
\usepackage{url}                
\usepackage{tikz}
\usetikzlibrary{positioning, quotes} 
\usetikzlibrary{shapes.geometric}
\tikzset{every path/.append style = {arrows = -latex}}

\theoremstyle{plain}
\newtheorem{theorem}{Theorem}[section]
\newtheorem*{theorem*}{Theorem}

\newtheorem{proposition}[theorem]{Proposition}
\newtheorem{lemma}[theorem]{Lemma}

\theoremstyle{definition}

\newtheorem*{assumption}{Assumption}

\title{Multistep Inverse Is Not All You Need}

\author{Alexander Levine  \\
    alevine0@cs.utexas.edu \\
    The University of Texas at Austin
    \And
    Peter Stone \\
    pstone@cs.utexas.edu\\
    The University of Texas at Austin and Sony AI
    \And
    Amy Zhang \\
    amy.zhang@austin.utexas.edu\\
    The University of Texas at Austin 
    }

\begin{document}

\maketitle

\begin{abstract}
In real-world control settings, the observation space is often unnecessarily high-dimensional and subject to time-correlated noise. However, the \textit{controllable} dynamics of the system are often far simpler than the dynamics of the raw observations. It is therefore desirable to learn an encoder to map the observation space to a simpler space of control-relevant variables. In this work, we consider the Ex-BMDP model, first proposed by \cite{efroni2022provably}, which formalizes control problems where observations can be factorized into an action-dependent latent state which evolves deterministically, and action-independent time-correlated noise. \cite{lamb2022guaranteed} proposes the ``AC-State'' method for learning an encoder to extract a complete action-dependent latent state representation from the observations in such problems. AC-State is a \textit{multistep-inverse} method, in that it uses the encoding of the the first and last state in a path to predict the \textit{first} action in the path. However,  we identify cases where AC-State will fail to learn a correct latent representation of the agent-controllable factor of the state. We therefore propose a new algorithm, ACDF, which combines multistep-inverse prediction with a latent forward model. ACDF is guaranteed to correctly infer an action-dependent latent state encoder for a large class of  Ex-BMDP models. We demonstrate the effectiveness of ACDF on tabular Ex-BMDPs through numerical simulations; as well as high-dimensional environments using neural-network-based encoders. Code is available at \url{https://github.com/midi-lab/acdf}.
\end{abstract}
\section{Introduction}
\label{sec:intro}

In rich-observation decision-making domains, such as robotics, much of the information that the agent observes is irrelevant to any plausible control objective. To allow for efficient planning, it is therefore desirable to learn a compact latent state representation, containing only the information potentially relevant to planning. One approach to
this problem is to learn a \textit{control-endogenous latent representation} \citep{efroni2022provably,lamb2022guaranteed}. The intuition behind this approach is that the observations that an agent receives, such as images, may contain a large amount of irrelevant information (including time-correlated noise) which represents parts of the environment that the agent has no control over. By contrast, the agent-controllable dynamics of the system can in some cases be represented by a \textit{small number of states with deterministic transitions.} This representation allows for efficient planning, and provides a view of the world model than can be directly interpreted by humans. The use of such representations has shown success, for example, in learning robotic manipulation tasks from images in a noisy environment \citep{lamb2022guaranteed}. 

\cite{efroni2022provably} introduces the Ex-BMDP formalism to represent this kind of environment. An Ex-BMDP is a (reward-free) Markov Decision Process in which each \textit{observed} state $x \in \mathcal{X}$ can be factored into an agent-controllable \textit{endogenous} state $s\in \mathcal{S}$, which follows a deterministic transition function, and an \textit{exogenous} state $e \in \mathcal{E}$, which evolves stochastically, independently of actions. It is important to note that the observations $x$ in this model are \textit{not explicitly} segmented into factors: rather, in order to extract the controllable latent state $s$, we must \textit{learn} an encoder.\footnote{Additionally, $x$ can also have time-independent noise: the observation $x$ is sampled from a distribution that depends only on $s$ and $e$. It is assumed that the observation $x$ contains enough information to fully specify $s$ and $e$.}

Numerous prior works have proposed methods for discovering latent representations useful for planning~\citep{wang2022denoised,zhang2020learning,pathak2017curiosity,NEURIPS2020_26588e93}. We focus our attention on the multi-step inverse method \citep{lamb2022guaranteed,islam2023principled}, which is compelling due to its explicit theoretical justification. In brief, consider states $x_t$ and $x_{t+k}$ visited by a policy. If an encoder $\phi$ is learned, such that $\phi(x_t)$ and $\phi(x_{t+k})$ provide sufficient information to predict the \textit{first action} $a_t$ on the path between  $x_t$ and $x_{t+k}$, then the learned representation $s = \phi(x)$ is claimed to be a complete endogenous latent state representation, providing the necessary and sufficient information to infer the latent dynamics. 

This paper discusses cases where the multistep inverse method proposed by \cite{lamb2022guaranteed}, known as AC-State, will fail to discover control-endogenous latent dynamics of an Ex-BMDP, and proposes a method that provably succeeds for a very general class of  Ex-BMDP's. In particular, we identify two flaws with the AC-State method:
\begin{itemize}
    \item The maximum length of the segment $k$ between $x_t$ and $x_{t+k}$ required for multistep-inverse dynamics prediction in order to correctly learn the encoder can be much larger than claimed.
    \item If the dynamics are periodic, then multistep-inverse dynamics are insufficient for learning an appropriate encoder $\phi$, regardless of $k$.
\end{itemize}
We then propose a modified method, which we call \textbf{ACDF}, which fixes these issues. We show that any encoder which minimizes our loss function (on infinite samples) is \textit{guaranteed to be a control-endogenous latent representation}. Specifically, we:
\begin{itemize}
    \item Give a corrected formulation of the number of steps of multistep-inverse dynamics prediction required to learn an Ex-BMDP.
    \item Propose to use a \textit{latent forward dynamics loss}, to enforce that the learned endogenous states are in fact compatible with deterministic dynamics.
\end{itemize}
In addition to our theoretical claims, we show empirically that ACDF can produce a more accurate endogenous latent model in Ex-BMDPs exhibiting certain properties.

\section{Background and Motivating Example} \label{sec:background}
Here, we formally describe the Ex-BMDP model and AC-State algorithm, and provide a simple example in which AC-State fails:
\subsection{Ex-BMDP Model} \label{sec:exbmdp}
\begin{figure}[h]
\centering
\begin{tikzpicture}[node distance = 0.4cm and 1cm, minimum size = 1cm]
\node[draw=black,shape=circle](et){$e_t$};
\node[draw=black,shape=circle](etm)[left = of et]{$e_{t-1}$};
\node[draw=black,shape=circle](etp)[right = of et]{$e_{t+1}$};
\node[draw=none](etmm)[left = of etm]{...};
\node[draw=none](etpp)[right = of etp]{...};

\node[draw=black,shape=circle,fill=lightgray](xt)[below = of et, xshift=0.5cm]{$x_t$};
\node[draw=black,shape=circle,fill=lightgray](xtm)[left = of xt]{$x_{t-1}$};
\node[draw=black,shape=circle,fill=lightgray](xtp)[right = of xt]{$x_{t+1}$};

\node[draw=black,shape=rectangle](st)[below = of xt, xshift=-0.5cm]{$s_t$};
\node[draw=black,shape=rectangle](stm)[left = of st]{$s_{t-1}$};
\node[draw=black,shape=rectangle](stp)[right = of st]{$s_{t+1}$};
\node[draw=none](stmm)[left = of stm]{...};
\node[draw=none](stpp)[right = of stp]{...};

\node[draw=black,shape=circle,fill=lightgray](at)[below = of st, xshift=0.5cm]{$a_t$};
\node[draw=black,shape=circle,fill=lightgray](atm)[left = of at]{$a_{t-1}$};
\node[draw=black,shape=circle,fill=lightgray](atp)[right = of at]{$a_{t+1}$};
\node[draw=none](atmm)[left = of atm]{...};
\node[draw=none](atpp)[right = of atp]{...};

\path  (etmm) edge (etm);
\path  (etm) edge (et);
\path  (et) edge (etp);
\path  (etp) edge (etpp);

\path  (stmm) edge (stm);
\path  (stm) edge (st);
\path  (st) edge (stp);
\path  (stp) edge (stpp);

\path  (stm) edge (xtm);
\path  (st) edge (xt);
\path  (stp) edge (xtp);

\path  (etm) edge (xtm);
\path  (et) edge (xt);
\path  (etp) edge (xtp);

\path  (atmm) edge (stm);
\path  (atm) edge (st);
\path  (at) edge (stp);
\path  (atp) edge (stpp);

\end{tikzpicture}
\caption{Probabilistic graphical model of the Ex-BMDP transition dynamics, as described in Section \ref{sec:exbmdp}. Endogenous states $s_t$ are shown as squares to indicate that they are \textit{deterministic} functions of the previous endogenous states and actions. Observations $x_t$ and actions $a_t$ are shown in gray to indicate that they are observable. We do \textit{not} show dependencies that may determine the actions~$a_t$.}
\label{fig:exbmbmdpgraphical}
\end{figure}
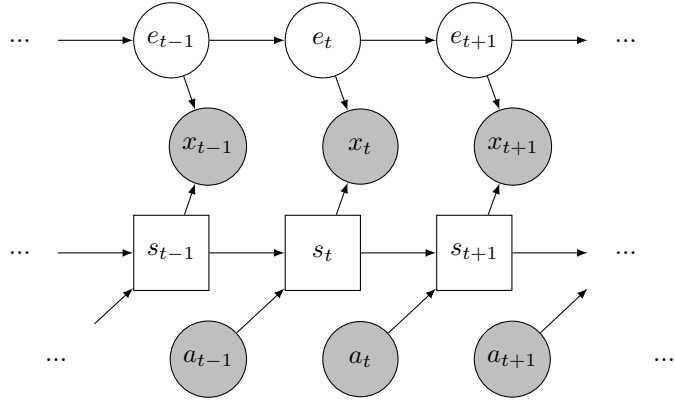
To formalize this notion of control-endogenous latent dynamics, consider a reward-free MDP with states $x \in \mathcal{X}$, discrete actions $a \in \mathcal{A}$, initial distribution $\mathcal{D}_0 \in \Delta(\mathcal{X})$,  and transition function $x_{t+1} \sim \mathcal{T}(x|x_t,a_t)$. The MDP admits a \textit{control endogenous latent representation} if its transition function can be decomposed as follows:
\begin{equation}
\begin{split}
    x_{t+1} &\sim \mathcal{Q}(x|s_{t+1},e_{t+1}),\\
    s_{t+1} &= T(s_{t}, a_{t}), \,\,\,\,\, s_t = \phi(x_t),\\
    e_{t+1} &\sim \mathcal{T}_e(e|e_{t}),\,\,\,\,\, e_t = \phi_e(x_t),\\
\end{split} \label{eq:def_exbmdp}
\end{equation}
where:
\begin{itemize}
    \item $s \in \mathcal{S}$ and $e \in \mathcal{E}$ are referred to as the \textit{control-endogenous} and \textit{control-exogenous} latent states, respectively. We assume that $\mathcal{S}$ is finite, and typically $ |\mathcal{S}| \ll |\mathcal{X}|$.\footnote{In our proofs, we also assume that $\mathcal{E}$ and $\mathcal{X}$ are finite, but this is a technical limitation of our theory: see discussion in Appendix \ref{sec:assumptions}. In any case, the Ex-BMDP formulation is most useful when $ |\mathcal{S}| \ll |\mathcal{X}|$.}
    \item The endogenous latent state $s$ evolves according to a \textit{deterministic} transition function $T$.
    \item The exogenous latent state $e$ evolves according to a Markovian transition function $\mathcal{T}_e$ that \textit{does not} depend on actions.
    \item We make the \textit{block assumption} on the observation emission function $\mathcal{Q}$ \citep{du2019provably}: that is, we assume that if $(s,e) \neq (s',e')$, then $\mathcal{Q}(x|s,e)$ and $\mathcal{Q}(x|s',e')$ have disjoint support. In other words, an observation $x \in \mathcal{X}$ corresponds to only a single pair $(s,e)$.
    \item The encoders $\phi,\phi_e$ are the (deterministic) inverses of $\mathcal{Q}$: that is, if $x \sim \mathcal{Q}(x|s,e)$, then $s= \phi(x),\,e=\phi_e(x)$.
\end{itemize}
An MDP that admits such a representation is known as an Ex-BMDP. Note that an Ex-BMDP can have \textit{multiple} valid factorizations into endogenous and exogenous states. (We discuss this fact further in Appendix \ref{sec:nonunique}.) In this work, our objective is to learn the encoder $\phi$ that is the endogenous state encoder for some valid factorization of the Ex-BMDP: specifically, we will aim to return a \textit{minimal-state} encoder: the number of endogenous states $|\mathcal{S}|$ should be as small as possible. We show a probabilistic graphical model of the Ex-BMDP transition dynamics in Figure \ref{fig:exbmbmdpgraphical}.
\subsection{AC-State} \label{sec:AC-State}
We now give more detail on the AC-State method as proposed by \cite{lamb2022guaranteed}. In addition to the Ex-BMDP formulation, that work makes the following further ``bounded diameter'' assumption:

\begin{assumption}[Assumption 3.1 from \cite{lamb2022guaranteed}] ``The length of the shortest path between any $z_1 \in \mathcal{S}$ to any $z_2 \in \mathcal{S}$  is bounded by $D$.''
\end{assumption}

\cite{lamb2022guaranteed} then proposes to learn the endogenous encoder $\phi_\theta$ by learning a \textit{multi-step  inverse dynamics model}. This model is a learned classifier $f(\phi_\theta(x_t),\phi_\theta(x_{t+k});k)$ which takes the encoded endogenous latent states of two observations separated by $k \leq D$ time-steps, as well as $k$, and returns a normalized distribution over \textit{actions} in $\mathcal{A}$. This classifier is trained to predict the first action $a_t$ taken on the trajectory from $x_t$ to  $x_{t+k}$.  The classifier $f$ is trained jointly with the encoder model. Then, in the theoretical treatment, the optimal encoder $\phi_{\theta^*}$ is defined as the encoder which allows $f$ to reach the minimum achievable value of this classification loss while also using the fewest number of distinct output states. (In practice, \cite{lamb2022guaranteed} uses a discrete information bottleneck and an associated loss term to minimize the output range, rather than learning multiple models.) Explicitly:
\begin{equation}
\begin{split}
    &\mathcal{L}_\text{AC-State}(\phi_\theta) := \min_f \mathop\mathbb{E}_{k \sim \{1,...,D\} }\mathop\mathbb{E}_{(x_t,a_t, x_{t+k})}  -\log(  f_{a_t}(\phi_\theta(x_t) ,\phi_\theta(x_{t+k}) ; k))\\
    &\{\theta\}^* := \{\theta^{**} | \theta^{**} = \arg\min_\theta \mathcal{L}_{\text{AC-State}}(\phi_\theta)\}\\
    &\theta^* := \arg\min _{\theta \in \{\theta\}^*} \|\text{Range}(\phi_\theta)\|\\
\end{split} \label{eq:multistep_inverse}
\end{equation}
where $f_{a_t}(\cdot)$ represents the probability assigned by $f$ to the action $a_t$. Given some assumptions about the behavioral policy (which are satisfied, for example, by a uniformly random policy), it is claimed that the $\phi_\theta$ which minimizes this loss (assuming perfect function approximation and large numbers of samples) will produce a control-endogenous latent representation of the MDP. The transition function $T(s,a)$ can then be inferred after-the-fact by applying the learned $\phi$ to all observed states $x$ and counting the transitions between the resulting latent states.

In practice, the diameter $D$ of the endogenous latent dynamics of the Ex-BMDP is unknown \textit{a priori}. Throughout this work, we will use $D$ to represent the true diameter of the endogenous latent dynamics\footnote{Specifically, the minimum diameter of \textit{any} endogenous representation which meets the assumptions in Appendix~\ref{sec:assumptions}.}, and $K$ to represent the number of steps actually used in practice. Based on the above assumption, AC-State is claimed by \cite{lamb2022guaranteed} to work as long as $K \geq D$.
\begin{figure}[h]
    \centering
    \includegraphics[width=\textwidth]{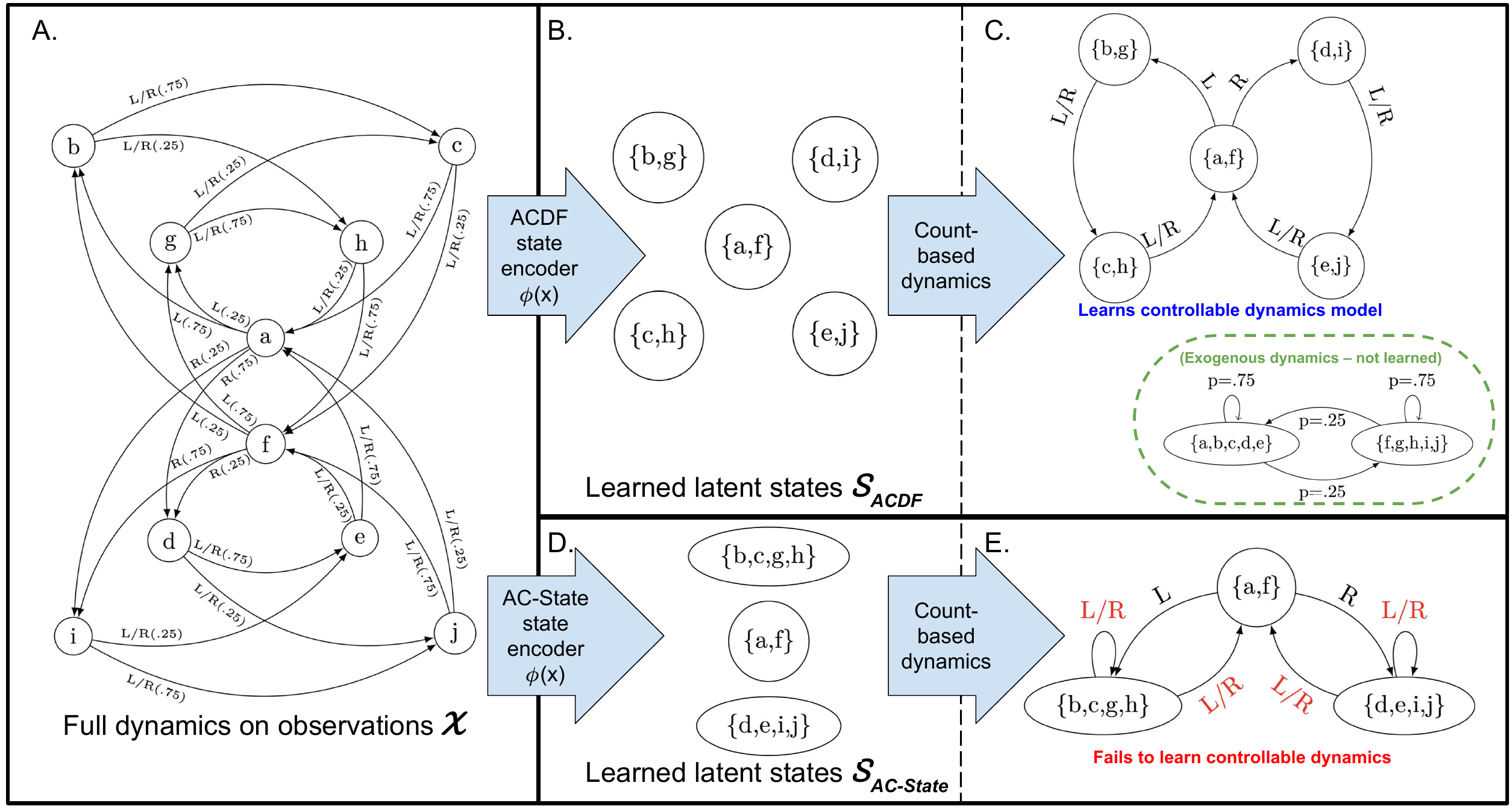}
    \caption{A tabular example where our proposed method ACDF successfully learns a control-endogenous state encoder, while the multistep-inverse method AC-State fails. (A) Full dynamics of the example Ex-BMDP: observed states are $\mathcal{X}=\{a,b...,j\}$ and actions are `L' and `R.' Transitions are stochastic: numbers in parentheses after action labels on transitions represent the probability of that transition, conditioned on the action.  (B) Encoded latent states $\phi(x) \in \mathcal{S}$, where $\phi$ is the encoder learned using our proposed method, ``ACDF.'' For example, $\phi$ maps the observed states $b$ and $g$ to the same latent state in $\mathcal{S}$. (C) Dynamics on the encoded latent states $\mathcal{S}$. The dynamics are deterministic, and capture the full agent-controllable factor of the state. Once $\phi$ is learned, these dynamics can be inferred from transition data by simple counting. The agent-\textit{independent} exogenous dynamics are shown in the inset: these dynamics are not learned by our method. (D) Encoded latent states produced by the encoder $\phi$ output by the AC-State algorithm \citep{lamb2022guaranteed}. (E)  The encoded latent states learned by AC-State are \textit{incorrect}: the encoding conflates states with different forward dynamics, resulting in under-determined transitions between latent states. }
    \label{fig:intro_figure}
\end{figure}

However, we can demonstrate a simple case where AC-State \textit{will not} recover a control-endogenous latent representation. In particular, consider the  Ex-BMDP shown in Figure \ref{fig:intro_figure}. If we focus our attention on the ``correct'' endogenous latent state dynamics shown in Figure \ref{fig:intro_figure}-C, we can note that:
\begin{itemize}
    \item We can't infer the action $a_t$ from $s_t$ and $s_{t+k}$ if $s_t$ is any state other than the state $s_t = \{a,j\}$, because for all other states in $\mathcal{S}_{\text{ACDF}}$, the actions L and R have the same effect.
    \item We can't infer $a_t$ given $s_t$ and $s_{t+k}$ for any $k > 2$, because, once $\{a,j\}$ is visited a second time, $a_t$ no longer has any impact on the current state.
    \item The only remaining case is that $s_t = \{a,j\}$ and $k\in \{1,2\}$. In this case, if $s_{t+k}$ is \textit{either} $\{b,g\}$ \textit{or}  $\{c,h\}$, then we know that $a_t =L$. Similarly, if $s_{t+k}$ is \textit{either} $\{d,i\}$ \textit{or}  $\{e,j\}$, then we know that $a_t =R$.
\end{itemize}
Note that there are \textit{no} cases where predicting $a_t$  requires distinguishing between the states $\{b,g\}$ and  $\{c,h\}$, or distinguishing between the states  $\{d,i\}$ and $\{e,j\}$. Therefore the optimal minimal-state AC-State encoder, the $\phi_{\theta^*}$ produced by Equation \ref{eq:multistep_inverse}, will distinguish between only three ``states'': $\{a,j\}$, $\{b,c,g,h\}$, and $\{d,e,i,j\}$.  However, these learned states do not constitute a control-endogenous latent representation of the MDP. In particular, the resulting transition function is not deterministic (See Figure \ref{fig:intro_figure}-E).
Consequently, the inferred control-endogenous state dynamics are not enough to predict, for example, whether taking the actions sequence $L,R,R,R$ starting in state $\{a,j\}$ will end at state $\{a,j\}$, state $\{b,c,g,h\}$, or state $\{d,e,i,j\}$.

This example demonstrates the non-universality of the multistep-inverse method at learning control-endogenous latent dynamics. In this work, we further develop the theory of endogenous latent dynamics, and demonstrate that combining the multistep-inverse method with a \textit{latent forward dynamics model} is in fact sufficient to learn a control-endogenous latent encoder.

\section{Guaranteed Learning of Control-Endogenous Dynamics}
In this section, we propose a modified loss function to replace $\mathcal{L}_\text{AC-State}$, for which we prove (in Appendix \ref{sec:ACDF}) that any minimum is a correct control-endogenous latent representation. Further, we show that the minimum-range $\phi_\theta$ which minimizes our loss function is a \textit{minimal-state} control-endogenous latent representation. We call our method \textbf{ACDF}, or \textbf{AC}-State+\textbf{D}'+\textbf{F}orward.
The loss function is given as follows:

\begin{equation}
\begin{split}
    \mathcal{L}_\text{ACDF}(\phi_\theta) := &\min_f \mathop\mathbb{E}_{k \sim \{1,...,D'\} }\mathop\mathbb{E}_{(x_t,a_t, x_{t+k})}  -\log( f_{a_t}(\phi_\theta(x_t) ,\phi_\theta(x_{t+k}) ; k))  \\
    +&\min_g \mathop\mathbb{E}_{(x_t,a_t, x_{t+1})}  -\log( g_{\phi_\theta(x_{t+1})}(\phi_\theta(x_t), a_t)).
\end{split}
\label{eq:acdf}
\end{equation}
Where, relative to $\mathcal{L}_\text{AC-State}$:
\begin{itemize}
    \item We have replaced the upper-bound on the control-endogenous diameter, $D$, with a new quantity $D'$, to be defined below. Note that if we are only given the diameter $D$, we can use an upper-bound $D' := 2D^2+D$, which is tight up to a constant multiple in the worst case.
    \item We have added a \textit{latent forward dynamics model} $g$ over the learned latent states. This model takes the encoded endogenous latent state of an observation $x_t$, and the action $a_t$, and returns a normalized probability distribution over the (discrete) encoded latent states. The model is trained to predict the next latent state $\phi_\theta(x_{t+1})$, which should be a deterministic function of the previous state and the chosen action, and is optimized jointly with $\phi_\theta$. 
\end{itemize}
We explain the logic behind these two modifications below:
\subsection{Why $D$ Steps is Insufficient for Multi-Step Dynamics} \label{sec:witness_dist}

We assume, as \cite{lamb2022guaranteed} does, that the Ex-BMDP admits a latent representation such that the control-endogenous dynamics have diameter upper-bounded by some $D$, as described in Section \ref{sec:intro}. In order to introduce our alternative bound $D'$, we first define a few quantities. For a given endogenous representation with states $\mathcal{S}$, we define the \textit{witness distance} $W(a,b)$ between states $a, b \in \mathcal{S}$ as the minimum number of steps $k$ such that there exists some \textit{witness state} $c \in \mathcal{S}$ such that a path of length \textit{exactly} $k$ exists from $c$ to $a$ and also from $c$ to $b$ (See Figure \ref{fig:witness_dist_figure}-A). Note that this quantity may be infinite even for a bounded-diameter graph, if the dynamics are periodic (in which case $W(a,b) = \infty$, see Figure \ref{fig:witness_dist_figure}-E). The \textit{witness distance of an endogenous representation} $W(\phi)$ is the \textit{maximum finite} $W(a,b)$, for any pair of states $a,b$ in the endogenous representation. Finally, the quantity $D'$ is defined as any upper bound on the witness distance $W(\phi)$.\footnote{More precisely, it is defined as any upper bound on the \textit{minimum value} of the witness distance $W(\phi)$ over the set of $\phi$'s that are minimum-$|S|$, finite-diameter endogenous latent representations of the Ex-BMDP. }

\begin{figure}
    \includegraphics[width=0.95\textwidth]{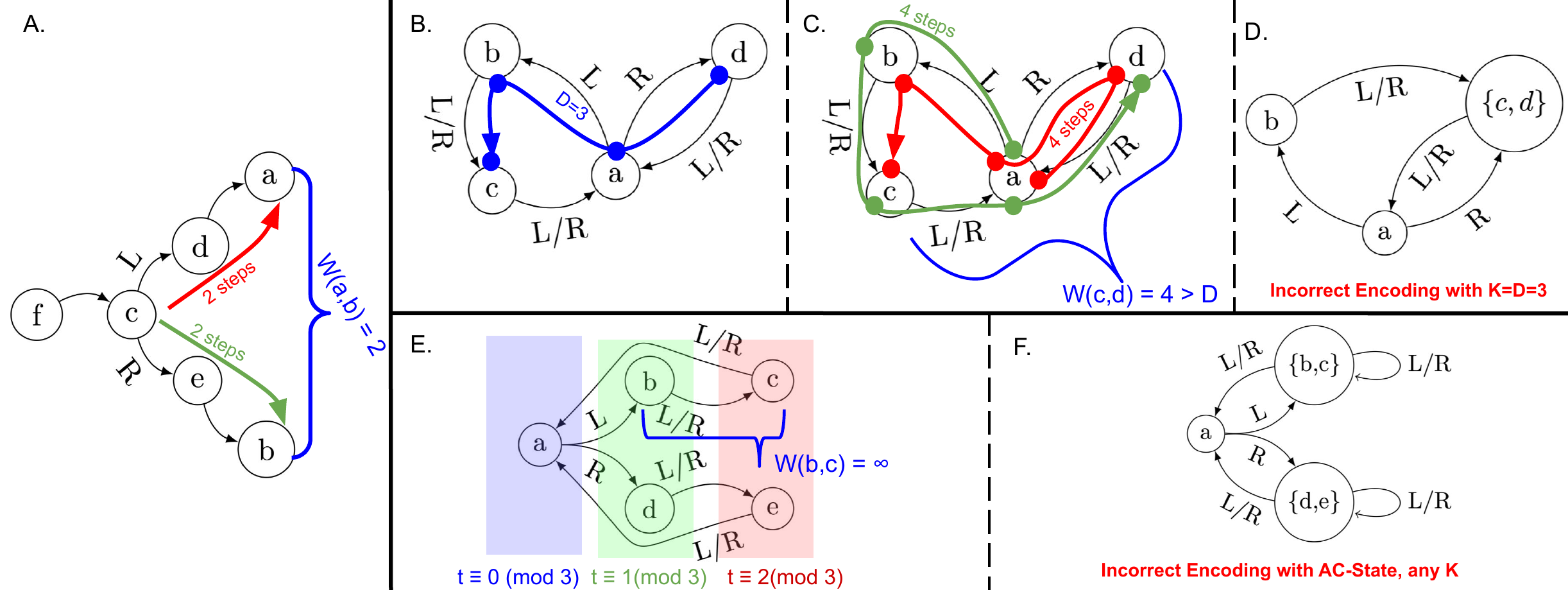}
    \caption{A. Example of the \textit{witness distance} $W(a,b)$. B-D. Witness distance can be greater than D, leading AC-State to fail. E-F. Witness distance can be \textit{infinite} if the dynamics are periodic, which also leads to AC-State failures. (See text of Section \ref{sec:witness_dist}.)}
    \label{fig:witness_dist_figure}
\end{figure}

The proofs in \cite{lamb2022guaranteed} implicitly assume that the witness distance $W(a,b)$ of any pair of states $a,b \in \mathcal{S}$ is upper-bounded by the diameter $D$ of the transition graph. To briefly sketch the main proof in \cite{lamb2022guaranteed}, let $W(a,b) = k \leq D$ with witness state $c$. Assume $s_t = c$ and  $s_{t+k} \in \{a,b\}$, and we wish to accurately predict $a_t$ given $x_{t}$ and $x_{t+k}$ (in order to minimize the $k$-step inverse loss, which is part of our loss function because $k \leq D$, by assumption). Being able to determine from $x_{t+k}$  whether  $s_{t+k} = a$ or $s_{t+k} =b$ is \textit{guaranteed to} help us make this prediction, because the sets of possible values that $a_t$ can take will be \textit{disjoint} depending on if  $s_{t+k} =a$ or $s_{t+k} =b$. (Otherwise, the witness distance would be most $k-1$, because $a$ and $b$ could both be reached from $s_{t+1}$, leading to a contradiction.)

 However, it is \textit{not} true in general that $W(a,b) \leq D$. Figure \ref{fig:witness_dist_figure}-B,C,D gives an example on a simple four-state graph, where $D=3$, but the largest witness distance $W(c,d) =4$. The AC-State loss with $K=D=3$ will learn an encoder that fails to distinguish all of the endogenous latent states. To analyse this example: we can't infer the action $a_t$ from $s_t$ and $s_{t+k}$ if $s_t$ is any state other than $a$, because for all other states, the actions L and R have the same effect. From $a$, the states reachable in exactly $k=1$ step are $\{b,d\}$; in $k=2$ steps are $\{c,a\}$; and in $k=3$ steps are $\{a,b,d\}$. In particular, when the classifier $f$ is provided with the value of $k \in \{1,2,3\}$, there is \textit{no need for the encoder to be able to distinguish} the state $c$ from the state $d$.

AC-State will therefore produce \textit{incorrect} learned endogenous latent dynamics for an Ex-BMDP with these true endogenous dynamics (Figure \ref{fig:witness_dist_figure}-D). This learned representation is incorrect because it is in fact \textit{controllable} whether the agent is in state $c$ or state $d$ four steps after being in state $a$ (as seen in \ref{fig:witness_dist_figure}-C), but the AC-State dynamics do not capture this controllablility.

It also may be \textit{impossible} to reach two states from the same state in exactly the same number of steps, leading to an \textit{infinite} witness distance. This case occurs if the endogenous dynamics are periodic. AC-State may then fail by conflating states belonging to different cyclic classes. (See Figure \ref{fig:witness_dist_figure}-E,F.)

\subsubsection{$D' := 2D^2+D$ is a Tight Upper-Bound}
In Appendix \ref{sec:d_bounds}, we show that if the witness distance $W(a,b)$ between any two states is finite, then it is upper-bounded by $D' := 2D^2+D$. Therefore, if all pairs of endogenous states have finite witness distance between them, and $D$ (or an upper-bound on $D$) is known, then the multistep-inverse loss with $K\geq 2D^2+D$ will be sufficient to distinguish all pairs of endogenous states, and hence learn a correct endogenous state encoder.

Furthermore, we show that this bound is tight up to a constant multiplicative factor. We explicitly construct Ex-BMDPs on which AC-State learns an incorrect state encoder for any $K <  D^2/2 +O(D)$, for an infinite sequence of arbitrarily-large values of $D$. If we use both the multistep-inverse loss and a latent forward-dynamics loss, as in the ACDF algorithm (Equation \ref{eq:acdf}), the particular family of  Ex-BMDPs  we use to derive this lower-bound no longer minimizes the loss. However, through an alternative construction, we can find Ex-BMDPs where the ACDF loss is minimized by an incorrect encoder if  $K <  D^2/4 +O(D)$. Therefore the upper bound  $D' := 2D^2+D$ is still tight for ACDF.

\subsection{Forward Latent Dynamics for Periodic Transition Functions}
The above discussion is applicable only to \textit{finite} witness distances. With bounded-diameter endogenous dynamics, we show in Appendix \ref{sec:d_bounds} that the witness distance between two states can be \textit{infinite}, if and only if the endogenous dynamics are \textit{periodic} (meaning that, for some period $p > 1$, each endogenous state $s\in \mathcal{S}$ can only be visited in time intervals that are multiples of $p$). Then AC-State may fail regardless of the number of steps K used in the multistep inverse dynamics. 

However, we show in Appendix \ref{sec:ACDF} that augmenting the multistep-inverse loss with a latent forward dynamics loss is sufficient to force the encoder to distinguish between states belonging to different cyclic classes of a periodic endogenous MDP. (In brief, either the states have different latent forward dynamics, or else they can be differentiated entirely by a cyclic \textit{exogenous} state factor, so the Ex-BMDP admits a more-minimal endogenous representation.) Thus, we prove that any $\phi$  which minimizes the ACDF loss is a correct endogenous latent state encoder.

\section{Experiments}
\label{sec:numerica}
\textbf{Numerical Simulation.} First, to capture the \textit{statistical} properties of our proposed method, without concerns about optimization or function approximation, we performed numerical simulations on tabular Ex-BMDPs. In these environments,  $|\mathcal{X}|$ is small enough that we can consider \textit{all possible} encoders $\phi$, and use count-based estimates for the classifier $f$. Consequently, the only source of error in minimizing $\mathcal{L}_{\text{AC-State}}$ or $\mathcal{L}_{\text{ACDF}}$ is the \textit{sampling error} caused by limited data collection. Full details of the experiments are included in Appendix \ref{sec:numerical_apdx}, and
results are shown in Figure \ref{fig:numeric_results}. In general, we see that the forward latent dynamics loss in $\mathcal{L}_{\text{ACDF}}$ not only enabled correct inference in periodic examples, but also lowered the value of $K$ necessary to learn the dynamics, and additionally yielded improved sample-efficiency even when it wasn't strictly necessary -- as in the ``control'' example.

\textbf{Deep Reinforcement Learning.} We also ran deep-RL experiments on two gridworld-like environments with image observations. The first ``baseline'' environment, from \cite{lamb2022guaranteed}'s released code, consists of nine copies of a four-room maze, with the ego-agent in one maze and random ``distractor'' agents in the others as exogenous noise. The second environment is constructed similarly, but with dynamics designed to be \textit{periodic}. Brief results for optimized hyperparameters are in Table \ref{tab:drl_results}, with more details and results in Appendix \ref{sec:DRL}. ACDF preserved performance on the baseline task while also more consistently learning a correct encoder for the periodic task.

\begin{figure}
    \centering
\includegraphics[width=0.93\textwidth]{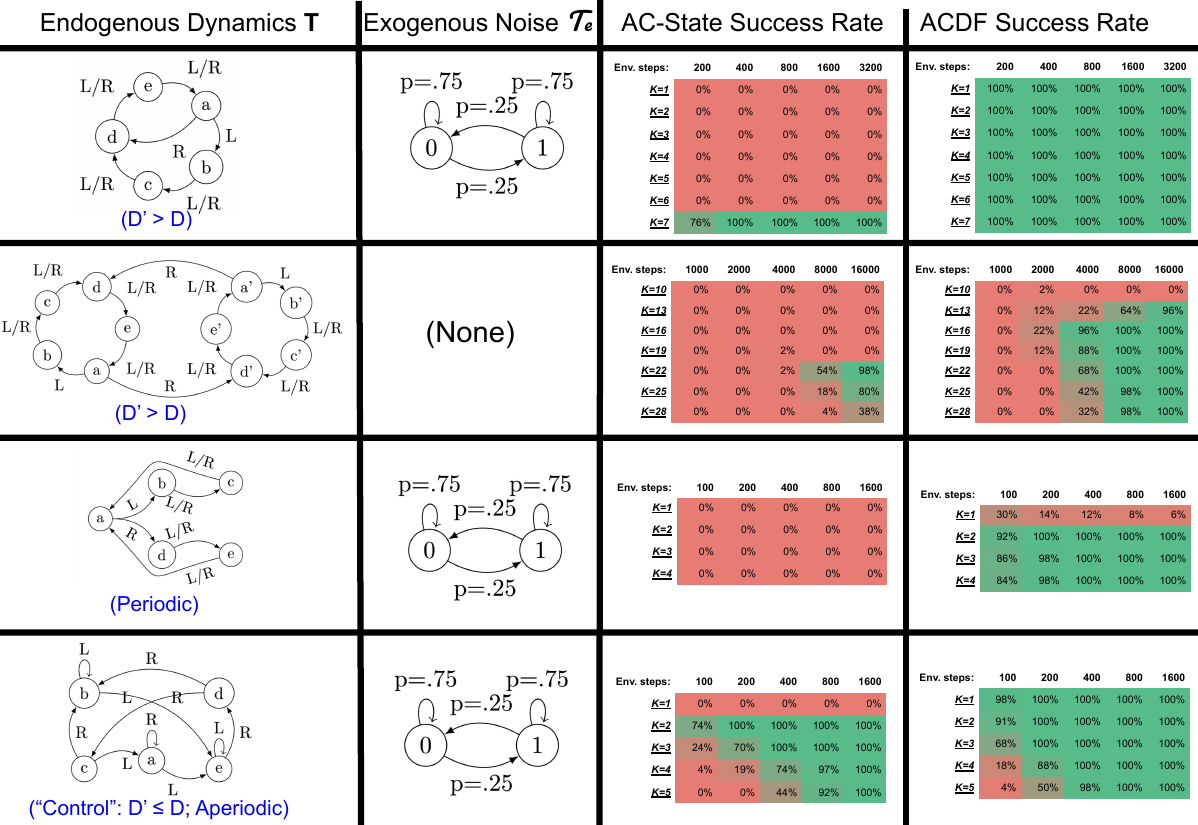}
    \caption{Results of numerical simulation experiments. Four environments are tested, with the dynamics given in the first two columns. For each environment, $|\mathcal{X}| = 10$, and $\mathcal{X}$ is isomorphic to  $\mathcal{S} \times \mathcal{E}$. In the last two columns, we show the success rate of each method (AC-State and ACDF) at learning the correct endogenous dynamics over 50 simulations. We show this success rate as a function of the hyperparameter $K$ and the number of environment steps used for learning.} 
    \label{fig:numeric_results}
\end{figure}
\begin{table}
    \centering
    \begin{tabular}{|c|c|c|c|c|}
    \hline
        &Baseline/AC-State &   Baseline/ACDF &  Periodic/AC-State &   Periodic/ACDF \\
        \hline
        Success Rate & 20/20 training runs& 20/20 '' ''& 1/20 '' ''& 19/20 '' ''\\
         \hline
    \end{tabular}
    \caption{Deep RL Results. Success measured as usability of the final $\phi$ for open-loop planning. }
    \label{tab:drl_results}
\end{table}

\section{Related Works} \label{sec:related_works}
In the area of learning latent representations for reinforcement learning problems, there are numerous ways of defining the ``purpose'' of the representation -- i.e., what information should ideally be included or excluded from the representation. Our work specifically focuses on the Ex-BMDP formulation given in Section \ref{sec:intro}, which was first proposed by \cite{efroni2022provably}, and later studied by \cite{lamb2022guaranteed}. Note that the setting considered by \cite{efroni2022provably} is \textit{time-inhomogeneous}: the Ex-BMDP is assumed to progress for a finite number of steps $H$ from a (near) deterministic control-endogenous start state, and the algorithm learns a \textit{different} state representation for \textit{each} time-step. Consequently,  number of states needed to represent the dynamics is potentially greatly increased, and this setting does not allow for generalization to long sequences, or encoding when the current time step is not known. In contrast, \cite{lamb2022guaranteed} and this work consider the infinite-time-horizon Ex-BMDP. Note that because AC-State is the only prior work to our knowledge that considers the infinite-time-horizon finite-$\mathcal{S}$ Ex-BMDP explicitly, we compare to it as our sole baseline in experiments. (While \cite{wang2022denoised} considers a similar formulation, with continuous latent states $\mathcal{S}$, the approach proposed in \cite{wang2022denoised} requires learning a \textit{generative model} of the \textit{complete} system state $\mathcal{X}$, explicitly modeling even the \textit{exogenous} state $\mathcal{E}$ of the system. \cite{kooi2023interpretable} also requires explicitly modeling both the ``controllable'' and ``uncontrollable'' latent dynamics, and also uses a somewhat different definition for how these factors relate than in the Ex-BMDP framework.) 

Other works present multi-step inverse methods learning for \textit{timestep-dependent} dynamics, including by \cite{mhammedi2023representation}, which allows for nondeterministic endogenous dynamics, and \cite{efroni2022sample}, which considers an explicitly factored state. Multi-step inverse methods have also been used in the continuous-state setting. \cite{mhammedi2020learning}  considers the special case of linear control-endogenous dynamics, and derives theoretical guarantees in this setting. More recently, \cite{islam2023principled} and \cite{koul2023pclast} have used multi-step inverse methods as an empirical technique for representation learning under continuous latent states. Note that while \cite{koul2023pclast} does learn a forward dynamics model, the forward-dynamics loss is not used to train the state encoder: it is only used for planning.

Aside from the Ex-BMDP framework and multi-step inverse methods,  other methods have been proposed to learn compact relevant state representations. \cite{efroni2022provably} discusses how several of these classes of methods will sometimes include control-exogenous noise into the latent representation, or fail to include control-endogenous information. Techniques such as Deep Bisimulation for Control \citep{zhang2020learning} and DeepMDP \citep{DBLP:conf/icml/GeladaKBNB19} ultimately rely on an external reward signal to determine what features are relevant for control, and so may fail to represent controllable aspects of the environment that do not affect the training reward. \cite{misra2020kinematic} demonstrates that one-step inverse dynamics, such as those used empirically by \cite{pathak2017curiosity}, are insufficient for learning control-endogenous dynamics. Techniques based on ``compressing'' states, e.g. through auto-encoders \citep{hafner2019learning} may also include exogenous information.

\cite{hutter2022uniqueness} is another work that examines the limitations of inverse models. However, it explores the conditions under which an inverse dynamics model on an MDP is sufficient to uniquely learn the transition function. It does not consider endogenous state representations or the ``block'' setting. Note that the problem of learning an endogenous encoder is somewhat ``easier'' than learning a forward dynamics model: in the setting we consider, the forward dynamics on the endogenous states are to be learned directly from samples of the transition function \textit{after} learning the encoder. By contrast, \cite{hutter2022uniqueness} considers the problem of inferring the transition function directly from the inverse dynamics \textit{alone} -- the conditions in which these tasks are possible may differ. 

\section{Limitations and Future Work}
One limitation to our work is that the cases where  AC-State fails are in some sense ``edge cases.'' In particular, if the endogenous dynamics have even a single state with a transition to itself, then AC-State with $K=D$ should succeed (see  Appendix \ref{sec:self_edges}.) However, in many real-world environments, staying in place forever is not possible, so ACDF may still be useful. In addition, we have shown that adding a forward dynamics loss can improve sample efficiency and reduce the dependence on $K$, even when ACDF is not strictly necessary. A general limitation of the Ex-BMDP model is that it assumes full  observation of the state. However, \cite{wu2023agentcentric} has extended AC-State to partially-observed environments: this extension can be straightforwardly adapted to ACDF. In Appendix \ref{sec:alt_fixes}, we discuss some alternative approaches we initially considered for ``fixing'' AC-State, which we determine do not work generally; however, this discussion may provide inspiration for future algorithms. 

\section*{Acknowledgements}
A portion of this research has taken place in the Learning Agents Research Group (LARG) at UT Austin. LARG research is supported in part by NSF (FAIN-2019844, NRT-2125858), ONR (N00014-18- 2243), ARO (E2061621), Bosch, Lockheed Martin, and UT Austin’s Good Systems grand challenge. Peter Stone serves as the Executive Director of Sony AI America and receives financial compensation for this work. The terms of this arrangement have been reviewed and approved by the University of Texas at Austin in accordance with its policy on objectivity in research. Alexander Levine is supported by the NSF Institute for Foundations of Machine Learning. We thank Alex Lamb for releasing the code for the original AC-State paper \citep{lamb2022guaranteed} at our request.
\bibliography{main}
\bibliographystyle{rlc}

\newpage
\appendix

\section{A Note on Terminology}
Due to assumptions (1), (4), and (5) in the technical assumptions section below (Section \ref{sec:assumptions}), when discussing MDPs (including Ex-BMDPs) in these appendices, we will most often be referring to finite MDPs under stationary behavior policies that assign nonzero probability to each possible action. As a consequence, properties of the induced Markov chain like periodicity, irreducibly, cyclic classes, etc. will be invariant to the particular choice of behavior policy, and will only depend on the structure of the MDP. Therefore, when unambiguous, we will sometimes refer to, for instance,``the periodicity of the MDP'', as shorthand to mean ``the periodicity of the induced Markov chain of the MDP under any policy that assigns nonzero probability to each action.'' (We avoid this shorthand in major theorem statements, such as the statement of Theorem \ref{thm:finite_wd}.)

\section{Technical Assumptions} \label{sec:assumptions}
In this section, we discuss the technical assumptions we make on the Ex-BMDP model, the behavioral policy, and on the collection of data. We also note where these assumptions differ from those of \cite{lamb2022guaranteed}.
Recall that, by the definition of the Ex-BMDP, there must exist \textit{at least one} endogenous state representation, which we call $s^* \in \mathcal{S}^*$, with deterministic dynamics $T^*$ and a corresponding exogenous state representation, which we will call $e^* \in \mathcal{E}^*$. Let the endogenous encoder $\phi^*$  map $\mathcal{X}$ to $\mathcal{S}^*$, and exogenous encoder $\phi_e^*$  map $\mathcal{X}$ to $\mathcal{E}^*$. Specifically, we will use these symbols to refer to some correct endogenous state representation which has minimal $|\mathcal{S}|$ among the set of correct endogenous state representations.

While this representation is not necessarily unique (as we show in Appendix \ref{sec:nonunique}) it will still be useful to refer to it. In particular, some of our assumptions are in terms of this representation. When we refer to an assumption about $\mathcal{S^*}$,  $\mathcal{E^*}$, etc., unless otherwise specified, we mean that there exists at least one minimal-endogenous-state decomposition of the Ex-BMDP for which all of these assumptions simultaneously hold.

We first give a brief statement of our assumptions, and then re-state them with more thorough discussion.

\textbf{Brief Assumptions:}
\begin{enumerate}
    \item In the proofs of the correctness of ACDF, we assume that \textbf{$\mathcal{X}$ is finite.} 
    \item We assume that \textbf{for some correct minimal-state endogenous encoder $\phi^*$, the endogenous latent dynamics $T^*$ have diameter bounded by $D$}.
    \item We assume that, for the exogenous encoder $\phi^*_e$ corresponding to the above $\phi^*$, \textbf{there are no transient exogenous latent states in $\mathcal{E^*}$}.
    \item We assume that, for the above-mentioned endogenous encoder $\phi^*$, \textbf{the behavioral policy $\pi(x)$ used to collect data depends only on $\phi^*(x)$, and ignores exogenous noise $\mathcal{E}^*$.} 
    \item \textbf{Coverage assumptions:}  \textbf{ For $k \leq D', x,x' \in \mathcal{X},a\in\mathcal{A}$, if $x'$ is reachable from $x,a$ in exactly $k$ steps, then we assume that we will sample $(x_t = x, a_t = a, x_{t+k} = x')$ with fixed, finite, nonzero probability.} Let $\mathcal{D}_{(k)}$ be the  distribution with this property from which $(x_t = x, a_t = a, x_{t+k} = x')$ are sampled when computing the expectation in Equation \ref{eq:acdf}.
    \end{enumerate}
We now discuss each of these assumptions in more detail:
\begin{enumerate}
    \item In the proofs of the correctness of ACDF, we assume that \textbf{$\mathcal{X}$ is finite.} This assumption is necessary because otherwise, some observations $x\in \mathcal{X}$  occur with infinitesimal probability, so an incorrect encoding $\phi(x)$ on such observations would not have a finite effect on the overall loss. (However, in practice, $\phi$ is a neural network with the capacity to generalize to new observations in a continuous space: this requirement should be thought of as more of a technical limitation of our proofs rather than a real constraint.) 
    \item Similar to the ``bounded diameter'' assumption of \cite{lamb2022guaranteed}, we assume that \textbf{for some correct minimal-state endogenous encoder $\phi^*$, the endogenous latent dynamics $T^*$ have diameter bounded by $D$}.
    \item We assume that, for the exogenous encoder $\phi^*_e$ corresponding to the above $\phi^*$, \textbf{there are no transient endogenous latent states in $\mathcal{E^*}$}. This assumption is necessary because it might not be possible to uniquely determine the endogenous state of an observation if the observation's exogenous state only occurs in the first few steps of a trajectory. (For example, if two endogenous states have the same forward dynamics, it would be impossible to uniquely assign an observation $x$ to one state or the other, if $\phi^*_e(x)$ only occurs at $t=0$.) Note that as a consequence of this assumption, the states of $\mathcal{E}^*$ can be partitioned into some number of closed, recurrent communicating classes, which each may be periodic or aperiodic. Additionally, because $\mathcal{X}$ is finite, we can assume that $\mathcal{E^*}$ is also finite. 
    \item We assume that, for the above-mentioned endogenous encoder $\phi^*$, \textbf{the behavioral policy $\pi(x)$ used to collect data depends only on $\phi^*(x)$, and ignores exogenous noise $\mathcal{E}^*$.} \cite{lamb2022guaranteed} also makes this assumption explicitly.  While this assumption may seem difficult to meet, because it seemingly requires prior knowledge of $\phi^*$, we note that a policy that takes actions in $\mathcal{A}$ according to a fixed distribution at all time-steps (such as a \textbf{uniform random policy}) will meet this requirement.
    \item \textbf{Coverage/Initial Distribution assumptions:} We only provide an asymptotic, rather than statistical, analysis. In other words, our results are only proven to hold for the \textit{population} expectation in Equation \ref{eq:acdf}, which occurs in the limit as samples approach infinity. However, me must still be explicit about the distribution from which we draw tuples $(x_t,a_t, x_{t+k})$ (and $(x_t,a_t, x_{t+1})$). To summarize, \textbf{ for $k \leq D', x,x' \in \mathcal{X},a\in\mathcal{A}$, if $x'$ is reachable from $x,a$ in exactly $k$ steps, then we assume that we will sample $(x_t = x, a_t = a, x_{t+k} = x')$ with fixed, finite, nonzero probability.} We discuss some implications below:
    \begin{itemize}
        \item \textbf{Unlike \cite{lamb2022guaranteed}, we do not assume a single trajectory}. \cite{lamb2022guaranteed} at one point mentions that data is assumed to be collected in a single, long trajectory, which follows dynamics that have a stationary distribution which assigns finite probability to all \textit{endogenous} states. However, such a sampling process does not necessarily give adequate coverage, because neither we nor \cite{lamb2022guaranteed} explicitly assume that the \textit{overall} dynamics $\mathcal{T}$ on the observations $\mathcal{X}$ are irreducible. In particular, $\mathcal{T}$ may be reducible if either the exogenous dynamics $\mathcal{T}^*_e$ are reducible, or if both $\mathcal{T}^*_e$ and the endogenous dynamics are periodic with the same period (a case we discuss further below). \textbf{We therefore assume at least one trajectory for each communicating class in $\mathcal{X}$.} Because  $\mathcal{X}$ is assumed finite, we can assume a finite number of these communicating classes. When we consider the ``infinite sample'' limit, we can assume that the length of \textit{each} of these trajectories goes to infinity (and therefore the time-averaged state visitation approaches a stationary distribution).
        \item Because (under any policy that assigns nonzero probability to every action) the endogenous dynamics are irreducible while the exogenous dynamics only have recurrent communicating classes,  it follows that each communicating class of $\mathcal{X}$  is finite, closed and recurrent. Therefore \textbf{any stationary policy that assigns nonzero probability to each action while in each endogenous state $s^* \in \mathcal{S^*}$ will, averaging over time, approach a stationary distribution assigning finite, nonzero probability to each $x$ in the communicating class in $\mathcal{X}$ that the chain started in.} Because chains will reach each $x$ with finite, nonzero probability and the execute each sequence of actions of finite length with finite probability, they \textbf{will therefore sample each possible $(x_t = x, a_t = a, x_{t+k} = x')$ with fixed finite probability. Any collection of such chains is a therefore valid data-collection mechanism.}
        \item \textbf{Unlike \cite{lamb2022guaranteed}, we do not assume that the initial exogenous and endogenous states are independent.} \cite{lamb2022guaranteed} assumes that the initial distribution of the Ex-BMDP $\mathcal{D}_0$ is such that $\phi^*(x_0)$ and $\phi_e^*(x_0)$ are distributed independently. 
        This assumption seems overly-strict, given that, as discussed above, \cite{ lamb2022guaranteed} \textit{also} assumes that collecting a single trajectory is sufficient. (If a single episode is sufficient to learn the dynamics, how can the \textit{distribution} of $x_0$ possibly matter?)
        We do \textit{not} make this independence assumption, rather, \textbf{we define $\mathcal{D}_0$ to be a distribution over $\mathcal{X}$ directly}.\footnote{We do, however, assume that  $x_0$ is in fact in the support of $Q^*(\phi^*(x_0), \phi_e^*(x_0))$; in other words, that any observation $x$ in the support of $\mathcal{D}_0$ is also in the support of $Q^*(s^*,e^*)$ for \textit{some} choice of $(s^*,e^*)$.} A particular consequence of \textit{not} making this independence assumption is that \textbf{we do not assume that all pairs ($s^*,e^*$) necessarily correspond to an observation in $\mathcal{X}$}. In other words, \textbf{we assume coverage over $\mathcal{X}$, not over $\mathcal{S}^* \times \mathcal{E}^*$.} This consequence comes into play when both the endogenous and exogenous dynamics are periodic, with the same period. We discuss the implications in Section \ref{sec:no_indep} below.
        \item \textbf{A note about periodicity in \cite{lamb2022guaranteed}:} One of the major claims of this paper is that \cite{lamb2022guaranteed} incorrectly handles Ex-BMDPs with periodic endogenous dynamics. Note that \cite{lamb2022guaranteed}  makes the explicit assumption that ``Markov chain $\mathcal{T}_D$ has a stationary distribution $\mu_D$ such that $\mu_D (s, a) > 0$ and $\pi_D (a | s) \geq \pi_\text{min}$, for all $s\in \mathcal{S}$ and $a\in \mathcal{A}$'', where ``$\mathcal{T}_D(s' | s)$ [is] the Markov chain induced on the control-endogenous state space by executing the policy $\pi_D$ by which AC-State collects the data.'' We wish to emphasise that this assumption \textit{should not} preclude Ex-BMDPs with periodic dynamics. In particular, note that periodic Markov chains (like $\mathcal{T}_D$ would be in the case of periodic dynamics) can indeed have stationary distributions that  assign nonzero probabilities to all states. Recall that a stationary distribution is merely any distribution $\mu$ for which $ \mu P = \mu$, where $P$ is the transition matrix. The existence of such a distribution is a property only of the transition function $P$, not the initial distribution. Therefore, even if a periodic Markov chain has a fixed initial state $x_0$, and therefore the distribution of the random variable $x_t$ (for any fixed $t$) will only have nonzero probability on states of one cyclic class, it is \textit{still the case} that the Markov chain can have a \textit{stationary distribution} that assigns probability to every state. Therefore \cite{lamb2022guaranteed} is not ``off the hook'' for periodic Ex-BMDPs. (Moreover, the periodicity or aperiodicity is never discussed explicitly in \cite{lamb2022guaranteed}, and the lines quoted above are part of a discussion about the behavioral policy $\pi_D$, \textit{not} about the structure of the Ex-BMDP. Therefore it is doubtful this assumption was intended in any way to exclude periodic dynamics.)

    \end{itemize}
    
    \end{enumerate}

\subsection{Implications of not assuming independence in initial exogenous and endogenous states} 
\label{sec:no_indep}
Above, we mentioned that, unlike \cite{lamb2022guaranteed}, we do not assume that the initial exogenous and endogenous states are distributed independently. This fact matters when both the endogenous and exogenous dynamics are periodic, with the same period. Assume that this period is $k$, and the exogenous dynamics are irreducible, such that $\mathcal{S}^*$ has cyclic classes $\{\mathcal{S}_0^*,...,\mathcal{S}_{k-1}^*\}$ and $\mathcal{E}^*$ has cyclic classes $\{\mathcal{E}_0^*,...,\mathcal{E}_{k-1}^*\}$. If the support of $\mathcal{D}_0$ only includes states $x_0$ such that, for example $(\phi^*(x_0), \phi^*_e(x_0))) \in \mathcal{S}^*_0 \times  \mathcal{E}^*_1 $, then the only states which will \textit{ever} be reachable will be states $x$ such that $(\phi^*(x), \phi^*_e(x)) \in \mathcal{S}^*_0 \times  \mathcal{E}^*_1  \cup \mathcal{S}^*_1 \times  \mathcal{E}^*_2   \cup ... \cup  \mathcal{S}^*_{k-1} \times  \mathcal{E}^*_0$. Because $\mathcal{X}$ is defined as the set of states that the Ex-BMDP can be in, it is therefore the case that $\forall x \in \mathcal{X}$, $(\phi^*(x), \phi^*_e(x)) \in \mathcal{S}^*_0 \times  \mathcal{E}^*_1  \cup \mathcal{S}^*_1 \times  \mathcal{E}^*_2   \cup ... \cup  \mathcal{S}^*_{k-1} \times  \mathcal{E}^*_0$. Note that in this case, $\mathcal{X}$ is \textit{not} equivalent to the union of the supports of $\mathcal{Q^*}(s^*,e^*)$, for all $s^*\in \mathcal{S^*} $ and  all $e^*\in \mathcal{E^*} $: in fact, in this case, $\mathcal{Q^*}(s^*,e^*)$ is not even well-defined for cases where $s^*$ and $e^*$ cannot co-occur. 

(Also, note that the \textit{initial distribution of the Ex-BMDP $\mathcal{D}_0$} is a distinct concept from the \textit{initial states used in the sampling trajectories.} The former is a property of the Ex-BMDP itself; the latter is a property of our algorithm.)

This difference in assumptions is important because there exist some Ex-BMDPs where a decomposition with \textit{independent} $(s_0,e_0)$ \textit{exists}, but a  more-minimal endogenous representation is possible if we allow for arbitrary initial distributions over $\mathcal{X}$. The ACDF algorithm will find these more-minimal representations. Explicitly, consider the following deterministic Ex-BMDP with $|\mathcal{X}| = 10$ and $\mathcal{A} = \{L,R\}$:

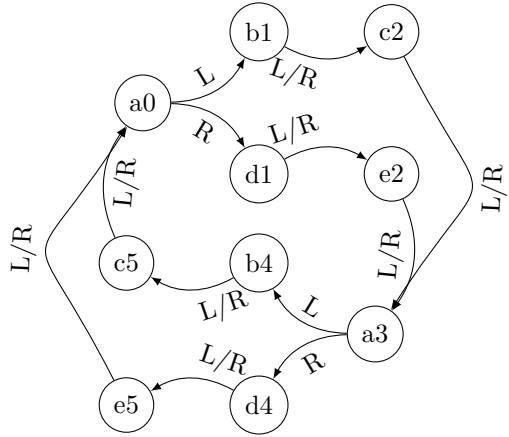
\begin{figure}[h!]
    \centering
\begin{tikzpicture}[node distance = 0.4cm and 1cm]
\node[draw=black,shape=circle](a){a0};
\node[draw=black,shape=circle](b)[above right = of a] {b1};
\node[draw=black,shape=circle](d)[below right = of a] {d1};
\node[draw=black,shape=circle](c)[right = of b] {c2};
\node[draw=black,shape=circle](e)[right = of d] {e2};
\node[draw=black,shape=circle](b')[below = of d] {b4};
\node[draw=black,shape=circle](a')[below right = of b']{a3};
\node[draw=black,shape=circle](d')[below left  = of a'] {d4};
\node[draw=black,shape=circle](c')[left = of b']{c5};
\node[draw=black,shape=circle](e')[ left = of d'] {e5};
\path  (a) edge [bend right]["L",sloped](b);
\path  (a) edge [bend left]["R",swap,sloped](d);
\path  (b) edge [bend right]["L/R",swap,sloped,pos=0.2](c);
\path  (d) edge [bend left]["L/R",sloped,pos=0.2](e);
\path  (c) edge [bend left]["L/R",sloped,looseness=2,pos=0.5, swap](a');
\path  (e) edge [bend left]["L/R",sloped](a');

\path  (a') edge [bend left]["L",sloped](b');
\path  (a') edge [bend right]["R",swap,sloped](d');
\path  (b') edge [bend left]["L/R",swap,sloped,pos=0.2](c');
\path  (d') edge [bend right]["L/R",sloped,pos=0.2](e');

\path  (e') edge [bend left]["L/R",sloped,looseness=2,pos=0.5](a);
\path  (c') edge [bend left]["L/R",sloped,swap](a);
\end{tikzpicture}
\caption{Full Ex-BMDP model of example in Section \ref{sec:no_indep}}. \label{fig:no_initial_independence}
\end{figure}

For this Ex-BMDP, let the initial state distribution $\mathcal{D}_0$ take value $x_0 = b1$ with probability 0.5, and $x_0 = e2$ with probability 0.5.

The ACDF algorithm will learn the  minimal-state endogenous representation of the Ex-BMDP shown in Figure \ref{fig:no_initial_independence_2}, with $|\mathcal{S} |= 5$.

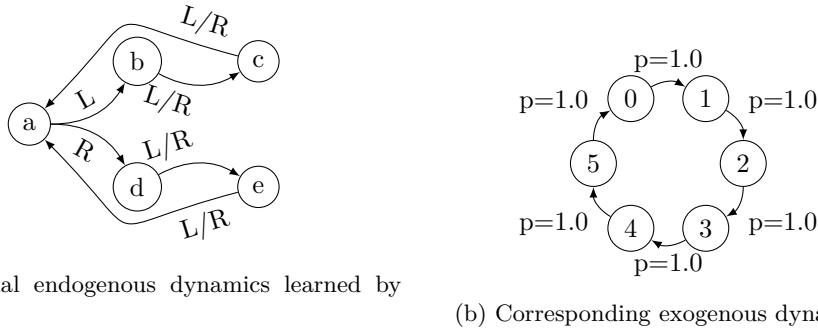
\begin{figure}[h!]
\centering
\begin{subfigure}{0.45\textwidth}
\centering
\begin{tikzpicture}[node distance = 0.4cm and 1cm]
\node[draw=black,shape=circle](a){a};
\node[draw=black,shape=circle](b)[above right = of a] {b};
\node[draw=black,shape=circle](d)[below right = of a] {d};
\node[draw=black,shape=circle](c)[right = of b] {c};
\node[draw=black,shape=circle](e)[right = of d] {e};
\path  (a) edge [bend right]["L",sloped](b);
\path  (a) edge [bend left]["R",swap,sloped](d);
\path  (b) edge [bend right]["L/R",swap,sloped,pos=0.2](c);
\path  (d) edge [bend left]["L/R",sloped,pos=0.2](e);
\path  (c) edge [bend right]["L/R",sloped,looseness=2,pos=0.1](a);
\path  (e) edge [bend left]["L/R",swap,sloped,looseness=2,pos=0.1](a);

\end{tikzpicture}
    \caption{Minimal endogenous dynamics learned by ACDF.}
\end{subfigure}
\centering
\begin{subfigure}{0.45\textwidth}
\centering
\begin{tikzpicture}
  \foreach \x in {0,...,5} {
      \pgfmathsetmacro{\angle}{(\x+1)*360/6}
    \pgfmathsetmacro{\xcoord}{-0.035cm*cos(\angle)}
    \pgfmathsetmacro{\ycoord}{0.035cm*sin(\angle)}
    \node[circle, draw, minimum size=0.5cm] at  (\xcoord,\ycoord) (nodec\x) {\x};}

\path  (nodec1) edge [bend left]["p=1.0"](nodec2);
\path  (nodec2) edge [bend left]["p=1.0"](nodec3);
\path  (nodec3) edge [bend left]["p=1.0"](nodec4);
\path  (nodec4) edge [bend left]["p=1.0"](nodec5);
\path  (nodec5) edge [bend left]["p=1.0"](nodec0);
\path  (nodec0) edge [bend left]["p=1.0"](nodec1);

\end{tikzpicture}
    \caption{Corresponding exogenous dynamics.}
\end{subfigure}
    \caption{Decomposition of the Ex-BMDP in Figure \ref{fig:no_initial_independence} that would be learned by the ACDF algorithm. As usual, the ACDF algorithm learns an encoder for the endogenous state  alone (left); the exogenous dynamics are implicit. (The observation $x$ emitted by latent states $(s,e)$ is given by concatenating the labels of $s$ and $e$; for example state $x=$ `c5' is reached when $s =$ `c' and $e =$ `5'.) }  \label{fig:no_initial_independence_2}
\end{figure}
This is a minimal endogenous latent representation, and one can confirm that the endogenous and exogenous latent dynamics together are equivalent to the dynamics shown in Figure \ref{fig:no_initial_independence} on any trajectory that starts on $\mathcal{D}_0$. However, under this decomposition, the initial state distribution $\mathcal{D}_0$ corresponds to a non-independent joint distribution on $\mathcal{S}$ and $\mathcal{E}$: it assigns probability 0.5 to the pair $(s_0=`b',e_0=`1')$ and probability 0.5 to the pair $(s_0=`e',e_0=`2')$. 

By contrast, consider the \textit{trivial} encoder $\phi'(x) = x$, where $\mathcal{S'} = \mathcal{X}$ and $\mathcal{E'}$ consists of a single state with a self-edge. This is also a valid control-endogenous representation (the transitions on $\mathcal{S'}$ are deterministic, and the (single-state) exogenous dynamics do not depend on actions.) Moreover, $s'_0$ and $e'_0$ are independent in  $\mathcal{D}_0$ under this decomposition. Therefore, if we were to accept  independent initial-state distribution as a \textit{defining feature} of the Ex-BMDP framework, we would be forced to return this trivial encoder with $|\mathcal{X}| = 10$, as the only valid Ex-BMDP decomposition.

We reject this assumption, because we believe that the more-minimal encoder returned by ACDF better captures an intuitive notion of \textit{controllability}: an agent has no control over whether they are in state `a0' or `a3' for instance. Further, while it might seem like the fact that $|\mathcal{S} \times \mathcal{E}|> |\mathcal{X}|$ means that our returned encoding is \textit{redundant}, recall that \textit{only the encoding of $\mathcal{S}$ is actually learned.}  In fact, loosening the definition of the Ex-BMDP model in this way can only lead to more concise learned representations (because it can only increase the set of valid decompositions of the Ex-BMDP).

Lastly, we note that, while \cite{lamb2022guaranteed} may \textit{state} an assumption of decoupled initial exogenous and endogenous states, AC-State also does not enforce in any way that the returned representation will have independent exogenous and endogenous states, even if such a representation is possible. (In fact, on this particular example, AC-State will return an \textit{incorrect} encoding consisting of 3 states with nondeterministic dynamics).
\subsubsection{Implications for the structure of $\mathcal{X}$}
As discussed above, our lack of an independence assumption on $s_0,e_0$ means that $\mathcal{X}$ may not contain elements corresponding to all pairs $(s,e)$, for $s \in \mathcal{S}$, $e \in \mathcal{E}$. However, we prove the following lemma which will be useful in showing when such an $(s,e)$ pair does correspond to some element in $\mathcal{X}$.
\begin{lemma}
    Consider any policy on $\mathcal{S}^*$ that assigns nonzero probability to all actions (i.e., any valid behavioral policy). Let $s, s' \in \mathcal{S}^*$ and  $e, e' \in \mathcal{E}^*$. If $(s',e')$ is reachable from $(s,e)$, then $(s,e)$ is reachable from $(s',e')$. Consequentially, if $(s',e')$ is reachable from $(s,e)$, and $(s',e')$ corresponds to an observation in $\mathcal{X}$, then  $(s,e)$ also corresponds to an observation in $\mathcal{X}$. \label{lemma:reachability}
\end{lemma}
\begin{proof}
    
Note that the dynamics on $\mathcal{S^*}$ are irreducible, and there are no transient states in $\mathcal{E^*}$. We know that $e$ and $e'$ must belong to the same communicating class in $\mathcal{E}^*$, so we can treat the dynamics on $\mathcal{E^*}$ as irreducible without loss of generality.

Let $M_s$ and $M_e$ be the two Markov chain transition matrices. Let $k_s$ be the periodicity of the endogenous dynamics, and $k_e$ be the periodicity of the exogenous dynamics. Note that $M_s^{k_s}$ is ergodic when restricted to the domain of each cyclic class of $\mathcal{S}^*$, and $M_e^{k_e}$ is ergodic when restricted to the domain of each cyclic class of $\mathcal{E}^*$.  Then for some $n$, $m$,  for any $n' \geq n$ and $m' \geq m$, $M_s^{k_s \cdot n'}$  has positive probability between any two states in the same cyclic class in $\mathcal{S^*}$ and $M_e^{k_e \cdot m'}$ has positive probability between any two states in the same cyclic class in $\mathcal{E}^*$. Then in particular, $M_s^{k_s \cdot k_e \cdot m \cdot n}$ has positive probability between any two states in the same cyclic class in $\mathcal{S^*}$, and  $M_e^{k_s \cdot k_e \cdot m \cdot n}$ has positive probability between any two states in the same cyclic class in $\mathcal{E^*}$.

Now, let the cyclic classes of $e$ and $s$ be defined as 0 on their respective Markov chains. Let $l((s,e),(s',e'))$ be the length of a path from $(s,e)$ to $(s',e')$. Now, consider any state $(s'',e'')$ reached by taking $k_e\cdot k_s - (l((s,e),(s',e'))\,\%\,(k_e\cdot k_s))$ steps starting at $(s',e')$. Then $(s'',e'')$ can be reached in a path of length $l((s,e),(s'',e''))$ from $(s,e)$, where 
\begin{equation}
    l((s,e),(s'',e'')) = l((s,e),(s',e') +  k_e\cdot k_s - (l((s,e),(s',e'))\,\%\,(k_e\cdot k_s)) \equiv 0 \pmod{k_e \cdot k_s}
\end{equation}
Then $s''$ belongs to the same cyclic class (0) as $s$ on $M_s$, and $e''$ belongs to the same cyclic class (0) as $e$ on $M_e$. Then $M_s^{k_s \cdot k_e \cdot m \cdot n}$ has positive probability to transition from $s''$ to $s$, and  $M_e^{k_s \cdot k_e \cdot m \cdot n}$ has positive probability to transition from $e''$ to $e$. Then we can reach $(s,e)$ from $(s'',e'')$ in a finite number ($k_s \cdot k_e \cdot m \cdot n$) of steps. Because we can reach $(s'',e'')$ from $(s',e')$, this implies that we can reach $(s,e)$ from $(s',e')$, as desired. As a consequence, if $(s',e')$ corresponds to an observation in $\mathcal{X}$, then the latent dynamics of the Ex-BMDP will eventually reach $(s,e)$ with positive probability, so $(s,e)$  also  corresponds to an observation in $\mathcal{X}$.
\end{proof}

\section{ Endogenous Latent Dynamics are not Unique} \label{sec:nonunique}
In this section, we provide some clarifications on the theory of endogenous latent dynamics, which  allow us to more clearly state our theoretical results. In their theoretical presentation, \cite{lamb2022guaranteed} implicitly assumes that a Ex-BMDP  has a \textit{unique} endogenous latent representation; i.e., that a single Ex-BMDP only admits a single (minimal $|\mathcal{S}|$) decomposition of the observation $x$ into endogenous state $s$ and exogenous state $e$.\footnote{ For example, \cite{lamb2022guaranteed} makes claims such as  ``We present an asymptotic analysis of AC-State showing it recovers the control-endogenous latent state encoder $f^*$.''} Consequentially, their theoretical claims are often made in terms of the ``ground truth'' endogenous latent state. Here, we show that this assumption is  unwarranted.

\subsection{An Ex-BMDP with Multiple Control-Endogenous Representations}
Consider the Ex-BMDP defined by control-endogenous states $\mathcal{S} = \{a,b\}$, exogenous states $\mathcal{E} = \{0,1\}$, actions $\mathcal{A} = \{\text{Stay},\text{Move}\}$ and transitions as follows:
\begin{figure}[h!]
\centering
\begin{subfigure}{0.45\textwidth}
\centering
\begin{tikzpicture}
\node[draw=black,shape=circle](a){a};
\node[draw=black,shape=circle](b)[right = of a]{b};
\path  (a) edge [bend right]["Move",swap](b);
\path  (b) edge [bend right]["Move"](a);
\path  (a) edge [loop above]["Stay"](a);
\path  (b) edge [loop above]["Stay"](b);
\end{tikzpicture}
    \caption{Endogenous transitions $T(s,a)$}
\end{subfigure}
\centering
\begin{subfigure}{0.45\textwidth}
\centering
\begin{tikzpicture}
\node[draw=black,shape=circle](0){0};
\node[draw=black,shape=circle](1)[right = of 0]{1};
\path  (1) edge [bend right]["p = 1.0",swap](0);
\path  (0) edge [bend right]["p = 1.0",swap](1);
\end{tikzpicture}
    \caption{Exogenous transitions $\mathcal{T}_e(e|e_t)$}
\end{subfigure}
    \label{fig:nonunique_1}
    \caption{An Ex-BMDP with $|\mathcal{S}| = 2$}
\end{figure}
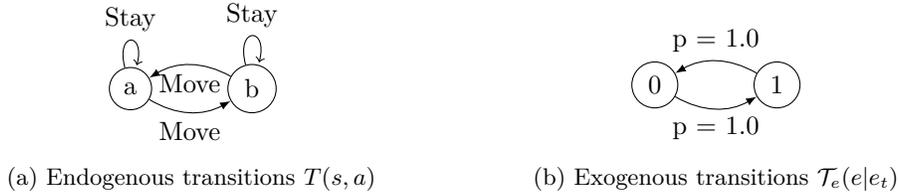

The observations $x \in \mathcal{X}$ are defined by simply concatenating the endogenous and exogenous state labels: $\mathcal{X} = \{a0,a1,b0,b1\}$. 

This Ex-BMDP in fact admits a different control-endogenous state representation, with the same (minimal) endogenous state space size $|\mathcal{S}| = 2$. In  particular, consider the representation:

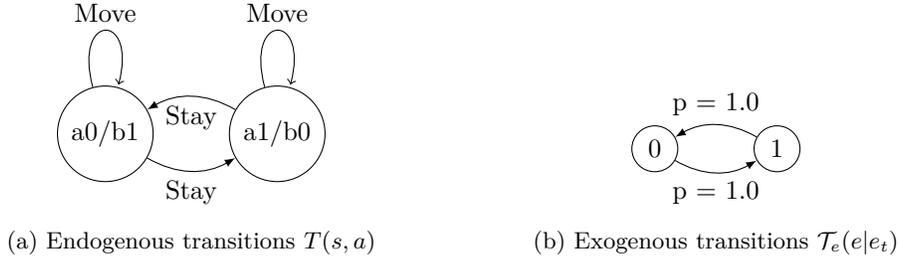
\begin{figure}[h!]
\centering
\begin{subfigure}{0.45\textwidth}
\centering
\begin{tikzpicture}
\node[draw=black,shape=circle](a0/b1){a0/b1};
\node[draw=black,shape=circle](a1/b0)[right = of a0/b1]{a1/b0};
\path  (a0/b1) edge [bend right]["Stay",swap](a1/b0);
\path  (a1/b0) edge [bend right]["Stay"](a0/b1);
\path  (a0/b1) edge [loop above]["Move"](a0/b1);
\path  (a1/b0) edge [loop above]["Move"](a1/b0);
\end{tikzpicture}
    \caption{Endogenous transitions $T(s,a)$}
\end{subfigure}
\begin{subfigure}{0.45\textwidth}
\centering
\begin{tikzpicture}
\node[draw=black,shape=circle](0){0};
\node[draw=black,shape=circle](1)[right = of 0]{1};
\path  (1) edge [bend right]["p = 1.0",swap](0);
\path  (0) edge [bend right]["p = 1.0",swap](1);
\end{tikzpicture}
    \caption{Exogenous transitions $\mathcal{T}_e(e|e_t)$}
\end{subfigure}
    \label{fig:nonunique_2}
    \caption{A different decomposition of the same Ex-BMDP, also with $|\mathcal{S}| = 2$}
\end{figure}

In this decomposition, the observation $x$ is defined such that, for example, if $s = a0/b1$ and $e= 0$, then $x= a0$. 

Note that this alternative encoding is not a mere relabeling of the same endogenous states: the pair of observations $\{a0,a1\}$ belong to \textit{different} endogenous states under this representation, although they have the same endogenous state under the first representation. We can confirm that the full MDP (that is, the observed MDP on $\mathcal{X}$) is in fact the same for both:
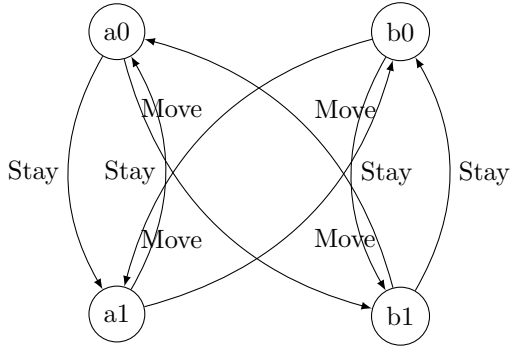
\begin{figure}[h!]
    \centering
\begin{tikzpicture}
\node[draw=black,shape=circle](a0){a0};
\node[draw=black,shape=circle](a1)[below =3cm of a0] {a1};
\node[draw=black,shape=circle](b0)[right =3cm of a0] {b0};
\node[draw=black,shape=circle](b1)[below =3cm of b0] {b1};

\path  (a0) edge [bend right]["Move",swap](b1);
\path  (a0) edge [bend right]["Stay",swap](a1);
\path  (a1) edge [bend right]["Move",swap](b0);
\path  (a1) edge [bend right]["Stay"](a0);
\path  (b0) edge [bend right]["Move",swap](a1);
\path  (b0) edge [bend right]["Stay"](b1);
\path  (b1) edge [bend right]["Move",swap](a0);
\path  (b1) edge [bend right]["Stay",swap](b0);
\end{tikzpicture}

\caption{Observed transitions on the full MDP}
    \label{fig:nonunique_3}
\end{figure}
\section{Bounds on D' and K} \label{sec:d_bounds}
In this section, we prove an upper bound on the maximum witness distance $D'$ in terms of the endogenous transition diameter $D$. We also show lower bounds on the number of steps $K$ of multi-step inverse needed to correctly learn a dynamics model in terms of $D$.
\begin{theorem}
Consider a  MDP on states $\mathcal{S}$ with deterministic transition function $T$, and finite diameter $D$ as defined by Assumption 3.1 from \cite{lamb2022guaranteed} (reproduced in Section \ref{sec:background} above). Recall the definition of witness distance $W(a,b)$ from Section \ref{sec:witness_dist}. For any $a,b \in \mathcal{S}$, either $W(a,b) \leq 2D^2 +D$, or $W(a,b) = \infty$. Furthermore,  $W(a,b) = \infty$ if and only if all Markov chains induced by any policy on $T$ (that assigns nonzero probability to each action) are periodic, and  $a$ and $b$ belong to different cyclic classes of such  chains. \label{thm:finite_wd}
\end{theorem}
\begin{proof}

Consider a Markov chain induced by any policy on the transition function $T$ that assigns nonzero probability to each action in each state. Because of the finite diameter, we know that the Markov chain is irreducible. We will initially consider the case that the Markov chain is aperiodic, and later reduce the general case to the aperiodic case.

Our proof technique is inspired by a related result from \cite{perkins1961theorem}, which demonstrated that, for an irreducible aperiodic Markov chain consisting of $N$ states, a path of length exactly $N^2 -2N +2$ exists between any pair of states. Building on this technique, our proof additionally takes advantage of the  bounded-diameter assumption present in our setting, instead of simply using $|\mathcal{S}|$. \cite{perkins1961theorem} uses the following theorem sourced from \cite{b2f296f6-cfc1-3d96-b687-b7c753783415}, which is ultimately attributed to Issai Schur (presented here in the form used in \cite{perkins1961theorem}):

\begin{theorem*}[Theorem of Schur]
    Consider any set of positive integers $\mathcal{B} = \{b_1,...,b_k\}$ such that $\text{gcd}(\mathcal{B}) = 1$, where $k = |\mathcal{B}|$,  $b_1$ is the smallest integer, and $b_k$ the largest. Then any integer $n \geq (b_1 -1)(b_k-1)$ can be represented as a sum in the form $n=\sum_{i=1}^k a_i b_i$, where $a_1,...,a_k$ are non-negative integers.
\end{theorem*}

For any state $a \in \mathcal{S}$, we know from the definition of aperiodicity that the greatest common denominator of the lengths of all possible self-loops from $a$ to $a$ must be 1. Let $\mathcal{L}_a$ be this complete set of self-loops of $a$. We will use $|\cdot|$ to denote length of a path (such as a loop). We will show first that there exists a subset $\mathcal{Q}_a \subseteq \mathcal{L}_a$, such that the length of each loop in $\mathcal{Q}_a$ is at most $2D+1$, the shortest loop in $\mathcal{Q}_a$ has length at most $D+1$, and such that  $\text{gcd}(\{|p| \big | p \in \mathcal{Q}_a\}) = 1$.

Consider the subset $\mathcal{P}_a \subseteq \mathcal{L}_a$  consisting of paths in the form $a \rightarrow ... \rightarrow a' \rightarrow  ... \rightarrow a$,  $ \forall a' \in  \mathcal{S} \setminus \{a\} $, where the segments  $a \rightarrow ... \rightarrow a'$ take a shortest-possible path from $a$ to $a'$, and conversely, the segments  $a' \rightarrow ... \rightarrow a$ take a shortest-possible path from $a'$ to $a$. In other words, $\mathcal{P}_a$ is the set of shortest self-loops of $a$ containing each possible other state $a'$. Let $d(a,a')$ be the length of the shortest path from $a$ to $a'$.  By the bounded-diameter assumption, we know that  $|p| = d(a,a') + d(a',a) \leq 2D, \forall p \in \mathcal{P}_a$. Furthermore, there must  exist some $a'$ with an edge directly incident on $a$: for this $a'$, we have $d(a',a) = 1$, and thus  $|p| = d(a,a') + d(a',a) \leq D +1$. Therefore $ \forall p \in \mathcal{P}_a,\,\,|p| \leq 2D$, and  $ \exists p \in \mathcal{P}_a,\,\,|p| \leq D+1$.

Now, there are two possibilities:
\begin{enumerate}
    \item $\text{gcd}(\{|p| \big{|} p \in \mathcal{P}_a\}) = 1$. In this case we tan take $\mathcal{Q}_a = \mathcal{P}_a $
    \item $\text{gcd}(\{|p| \big{|} p \in \mathcal{P}_a\}) = n > 1$. In this case, let $m$ be any factor of $n$ greater than 1. Because  $\text{gcd}(\{|p| \big{|} p \in \mathcal{L}_a\}) = 1$, there must exist some $p' \in \mathcal{L}_a$ such that $m \nmid |p'|$ ($m$ does not divide $|p'|$). (Furthermore, we can assume that the state $a$ does not occur in the middle of $p'$: if we are given such a ``double-loop" with  $m \nmid |p'|$, then $m$ must not divide the length of either the first segment or the second segment, or both.)  Let $p'_i$ be the $i$th element in the path $p'$ (zero-indexed). Then $p'_0 = p'_{|p'|} = a$. Note that 
    \begin{itemize}
        \item $ d(a,p'_{|p'|}) - d(a,p'_0) = 0-0 = 0 \equiv 0 \pmod{m}$
        \item $|p'| -0 = |p'| \not \equiv 0 \pmod{m}$
    \end{itemize}
    Therefore, it \textit{cannot} be the case that for all $i \in [0,|p'|)$, \,\,\, $d(a,p'_{i+1}) - d(a,p'_{i}) \equiv 1 \pmod{m}$. (To confirm this fact, sum both sides of this equivalence over $i$ from $0$ to $|p'|-1$.)
    Fix $i$ so that $p'_{|i|}$ is some state on $p'$ where this equivalence does not hold. (That is,   $d(a,p'_{i+1}) - d(a,p'_{i}) \not\equiv 1 \pmod{m}$.) Now, consider a self-loop of $a$ constructed from the following segments:
    \begin{itemize}
        \item A shortest-length path from $a$ to $p'_{i} $
        \item The edge from $p'_{i}$ to $p'_{i+1}$
        \item  A shortest-length path from $p'_{i+1} $ to $a$.
    \end{itemize}
Let this loop be known as $p''$, and note that $|p''| \leq 2D +1$, because it consists of two shortest paths and a single edge. Also, observe that:

\begin{align}
          d(a,p'_{i+1}) - d(a,p'_{i}) &\not\equiv 1 \pmod{m} \text{\,\,\,\,(As derived above)}\label{eq:DProof1}  \\
          d(a,p'_{i+1}) + d(p'_{i+1},a) &\equiv 0 \pmod{m} \text{\,\,\,\,(A self-loop in $ \mathcal{P}_a$)}\label{eq:DProof2}  \\
          -d(p'_{i+1},a) - d(a,p'_{i}) &\not\equiv 1 \pmod{m} \text{\,\,\,\,(Subtract Eq. \ref{eq:DProof2} from  Eq. \ref{eq:DProof1})} \label{eq:DProof3}  \\
         d(a,p'_{i}) +   d(p'_{i+1},a)  &\not\equiv -1 \pmod{m} \text{\,\,\,\,(Rearrange and negate Eq. \ref{eq:DProof3}.)} \label{eq:DProof4} \\
         d(a,p'_{i}) + 1 +  d(p'_{i+1},a)  &\not\equiv 0 \pmod{m} \text{\,\,\,\,(Add 1 to Eq. \ref{eq:DProof4}.)} \label{eq:DProof5}\\   
         |p''|  &\not\equiv 0 \pmod{m} \text{\,\,\,\,(From the definition of $p''$ and Eq. \ref{eq:DProof5}.)} 
\end{align}
Note that we can construct such a $p''$ for each factor $m$ of $n$. Then we can let $\mathcal{Q}_a$ consist of each loop in $\mathcal{P}_a$, and additionally one of these $p''$ loops constructed for for each factor $m$. Then by construction, $\text{gcd}(\{|p| \big{|} p \in \mathcal{Q}_a\}) = 1$ and $\forall p \in \mathcal{Q}_a, \,\,|p| \leq 2D+1$, as desired. (Additionally, $\mathcal{Q}_a$ still contains the loop of length $\leq D+1$ from $\mathcal{P}_a$, as desired.)
\end{enumerate}
Now, we have constructed  $\mathcal{Q}_a$ such that the longest loop in $\mathcal{Q}_a$ has length at most $2D+1$ , while the shortest loop has length at most $D+1$, and furthermore that the lengths of the loops are relatively prime. Then the Theorem of Schur given above tells us that by taking some combination of self-loops in  $\mathcal{Q}_a$ from $a$ in sequence, we can construct a path from $a$ to $a$ of any arbitrary length $n$, for any $n \geq (2D+1-1)(D+1-1) = 2D^2$. Then for any $b \in \mathcal{S}$, we can construct a path from $a$ to $b$ of length exactly $2D^2 + d(a,b)$ (by self-looping for exactly $2D^2$ steps at $a$, and then taking the shortest path to $b$), and also a path from $a$ to $a$ of length exactly $2D^2 + d(a,b)$ (because  $2D^2 + d(a,b) \geq 2D^2$, so such a self-loop must exist). Then $a$ can act as a witness for $a$ and $b$, and $W(a,b) \leq 2D^2 + d(a,b)  \leq 2D^2+D$, as desired.

Finally, we return to the periodic case. Let $k$ be the periodicity of the Markov chain, such that the chain is partitioned into the $k$ cyclic classes $\mathcal{S}_0, ..., \mathcal{S}_{k-1}$, where a state in class $i \pmod{k}$ is always succeeded by a state in class $i+1 \pmod{k}$. There are two possibilities:
\begin{itemize}
    \item States $a$ and $b$ belong to the same cyclic class,  which we can call $\mathcal{S}_{i_a}$. Then consider the Markov chain raised to the $k$'th power; that is, the Markov chain induced by taking $k$ steps of the original Markov chain at each step; call this new chain $M^k$. We will show that $M^k$ is irreducible, bounded-diameter, and aperiodic, and thus we can apply the above aperiodic case to $M^k$.
    First, consider any path starting in $\mathcal{S}_{i_a}$ in the original chain, of length $|p|$. Either $k \mid|p|$, in which case it ends in  $\mathcal{S}_{i_a}$ and has an equivalent path of length $|p|/k$ in $M^k$, or  $k \nmid|p|$, in which case it does not end in  $\mathcal{S}_{i_a}$  and also has no equivalent path in $M^k$. Thus we see that, if we start in $\mathcal{S}_{i_a}$,  $M^k$ forms an irreducible Markov chain on the states of $\mathcal{S}_{i_a}$ alone. Furthermore,   $M^k$  on the states $\mathcal{S}_{i_a}$ has diameter at most $\lfloor D/k \rfloor$ (because there is a path of length of most $D$ between every pair of states in $\mathcal{S}_{i_a}$ in the original Markov chain). Finally, $M^k$  on the states $\mathcal{S}_{i_a}$ is aperiodic: to confirm this fact, note that by the definition of periodicity, we have that in the original Markov chain, the g.c.d. of the lengths of all self-loops from any particular state in $\mathcal{S}_{i_a}$ is $k$. All of these self-loops will still exist in $M^k$, but their lengths will be divided by $k$. Therefore, the g.c.d. of the lengths of the self-loops will be 1. This fact directly implies that $M^k$ on $\mathcal{S}_{i_a}$ is aperiodic. 
    Now, because $M^k$ on $\mathcal{S}_{i_a}$ is irreducible and aperiodic with bounded diameter $\lfloor D/k \rfloor$, we can apply the aperiodic case above to $M^k$. Then, in   $M^k$ on $\mathcal{S}_{i_a}$, we have that $W(a,b) \leq 2\lfloor D/k \rfloor^2 + \lfloor D/k \rfloor \leq 2 D^2/k^2 +  D/k$.  The paths from the witness state to $a$ and $b$ will also exist in the original MDP, but will be longer by a factor of $k$. Therefore the witness distance in the original MDP will be  $W(a,b)  \leq 2 D^2/k +  D \leq 2 D^2 +  D$.

    \item States  $a$ and $b$ belong to different cyclic classes, which we can call $\mathcal{S}_{i_a}$ and  $\mathcal{S}_{i_b}$, where $i_a \not\equiv i_b \pmod{k}$. In this case, for any witness state $c$ in any $\mathcal{S}_{i_c}$, note that any path from $c$ to $a$ will have length congruent to $i_a -i_c \pmod{k}$, while any path from $c$ to $b$ will have length congruent to $i_b -i_c \pmod{k}$. Because  $i_a \not\equiv i_b \pmod{k}$, we have that $i_a - i_c\not\equiv i_b -i_c \pmod{k}$, so no path from $c$ to $a$ can be the same length as any path from $c$ to $b$. Then $W(a,b) = \infty$. Note that this case is the \textbf{only} case in which $W(a,b) = \infty$.
\end{itemize}
\end{proof}

The above result shows that the maximum witness distance $D'$ is upper-bounded by $2D^2+D$, and thus, if we are given diameter $D$, we can safely set the hyperparameter $K=2D^2+D$ and be guaranteed that ACDF will correctly discover a endogenous state representation.

We now show two kinds of lower bounds on $K$, the number of  multi-step inverse steps \textit{actually} needed to learn the endogenous state:
\begin{itemize}
    \item A bound for the AC-State loss. For this bound, we assume that the dynamics are aperiodic; thus $K=2D^2+D$ steps are in fact sufficient to find a correct endogenous state representation.
    \item A bound for the ACDF loss.
\end{itemize}
These results are proven by construction: we will explicitly construct examples of dynamics for which AC-State/ACDF fail when $K$ is too small.

We first give the following useful lemma:
\begin{lemma} \label{lemma:one-to-one}
    Consider an Ex-BMDP defined on states indexed as $\mathcal{X} = \{0,...,|\mathcal{X}|-1\}$, such that the dynamics of the Ex-BMDP are \textit{deterministic}, and such that any Markov chain induced by any policy on the Ex-BMDP (that puts nonzero probability on every action) is \textit{irreducible} and \textit{aperiodic}. Then the \textit{only} valid endogenous latent encoder (up to a relabeling permutations) is the trivial one: $\phi(x) = x$. That is to say, the endogenous state $s$ is the full state $x$, and the endogenous transition function is the same as the full Ex-BMDP transition function. 
\end{lemma}
\begin{proof}
We proceed as follows:
\begin{itemize}
    \item Consider any pair of endogenous state $s$ and exogenous state $e$. By the block assumption, each observed state $x$ corresponds to exactly one such pair, and any two observed states $x,x'$ which correspond to the same pair $(s,e)$ must have in-edges in $M$ from exactly the same set of states in $\mathcal{X}$, with exactly the same actions for each edge, with a fixed probability ratio. Because the Ex-BMDP is deterministic, there are no such pairs, and it immediately follows that each pair $(s,e)$ corresponds to exactly one $x$, and vice-versa.
    \item Because the full dynamics of $x= (s,e)$ are deterministic, and because the dynamics of $s$ are deterministic (by definition), it immediately follows that the dynamics of $e$ must be deterministic. 
    \item Given that the dynamics of $e$ are deterministic and by definition do not depend on actions, $\mathcal{E}$ and $\mathcal{T}_e$ can only consist of some set of disjoint cycles of states. However, if there are multiple such cycles, then under any choice of policy, any Markov chain on $x= (s,e)$ would have multiple communicating classes (because states $x =(s,e)$ with values of $e$ on different cycles would be inaccessible from each other), and would therefore not be irreducible. Therefore $\mathcal{E}$ can only consist of a single cycle. However, this cycle can in fact only consist of a single state: note that the length of any self-loop of any state $x = (s,e)$ will by divisible by the cycle length of $\mathcal{E} = |\mathcal{E}|$. Because the Markov chain is aperiodic, the g.c.d of such loops must be 1. Then $ |\mathcal{E}|$ must be equal to 1.
\end{itemize}
Therefore, each $x$ corresponds one-to-one with a endogenous-exogenous pair $(s,e)$, and there is only a single exogenous state $e$. So each $x$ corresponds one-to-one with an endogeous state $s$, as desired.
\end{proof}
We now present our lower-bound results:

\begin{proposition}
       $\forall D$, there exists an Ex-BMDP such that a minimal endogenous latent dynamics has a diameter of $D$ and is aperiodic, and such that the AC-State algorithm given by Equation \ref{eq:multistep_inverse} with a uniform exploration policy will return an encoder that does \textbf{not} produce a valid endogenous state representation using $K \leq h(D)$ -step inverse dynamics, where $h(D) \in \Omega(D^2)$. Here, we are using the Hardy-Littlewood $\Omega$ notation (not to be confused with Knuth's $\Omega$), and specifically $h(D)\sim D^2/2$, in the sense that:
     \begin{equation}
        \limsup _{D\rightarrow \infty} \frac{h(D)}{\frac{1}{2} D^2} = 1
     \end{equation} \label{prop:ACState_fail}
\end{proposition}
\begin{proof}
Consider two arbitrary primes $p,q$, with $p < q$, and let these primes define a deterministic Ex-BMDP $M$ with transition function $T(x,a)$ as follows:
\begin{itemize}
    \item $q$ states, labeled $\mathcal{X} = \{0,...,q-1\}$
    \item Two actions, $\{L,R\}$
    \item $T(s,a)$ defined as:
    \begin{itemize}
        \item $T(0,L) = 1$
        \item $T(0,R) = q-p+1$
        \item $\forall x\in \mathcal{X}, x\neq 0$: $T(x,L) = T(x,R) = (x+1)\,\%
        \,q$, where $\%$ is the modulo operator.
    \end{itemize}
\end{itemize}

We first note that for this Ex-BMDP, under any policy that assigns nonzero probability to all actions, the resulting Markov chain is irreducible and aperiodic. For irreducibility, simply note that any state can be reached from any other state (i.e., any state $b$ can be reached from any state $a$ by simply taking action L $(b-a)\,\%\,q$ times). For aperiodicity, note that the state $0 \in \mathcal{X}$ has a self-loop of length $q$ (by taking $L$ and then any other $q-1$ actions) and a self-loop of length $p$ (by taking $R$ and then any other $p-1$ actions). Because these are both prime, the g.c.d. of the lengths of all of state $0$'s self-loops is $1$, so the Markov chain is aperiodic.  Then, by applying Lemma \ref{lemma:one-to-one}, we know that the \textit{only} valid endogenous latent encoder (up to a relabeling permutations) is the trivial one: $\phi(x) = x$. We can thus regard the states $\mathcal{S} = \mathcal{X}$ as the only ``true'' endogenous states: any encoder which maps any two states in $\mathcal{X}$ to the same endogenous state must be incorrect. 

Now, consider the AC-State prediction task of determining $a_t$ from $\phi(x_t)$,  $\phi(x_{t+k})$, and $k$. We will assume that $\phi$ maps the state $0\in \mathcal{X}$ to its own endogenous state (otherwise, AC-State has \textit{already} failed, as desired).  

Note that under a uniform exploration policy, where defined,
\begin{equation}
   \forall x\in \mathcal{X}\setminus \{0\},\,  x'\in \mathcal{X},\, k\in \mathbb{N}\,\,\,\,\,\,\Pr(a_t = L | x_t=x,x_{t+k}=x') = 0.5 
\end{equation}
that is, it is impossible to determine $a_t$ from the inverse dynamics unless $x_t = 0$ (because for $x_t\neq 0$, the next action $a_t$ does not affect the dynamics). Then it is only useful to distinguish two values of $x \in\mathcal{X} \setminus \{0\}$ if they produce distinct values of $\Pr(a_t = L | x_t=0,x_k=x)$ for some \textit{particular} $k \leq K$. In particular, \textbf{if two states $x,x'\in\mathcal{X} \setminus \{0\}$ cannot each occur exactly $k$ steps after state 0 for some $k \leq K$, then the AC-State loss can be minimized without $\phi$ distinguishing $x$ and $x'$.} Because AC-State returns the minimum-range encoder $\phi$, it will preferentially return an encoder that maps $x$ and $x'$ to the same endogenous state over the only correct encoder, which maps every $x$ to its own $s$. Therefore, to show that AC-State will fail for some $K$, we need only show that there are two states $x,x'\in\mathcal{X} \setminus \{0\}$ which cannot both be reached from state $0$ in the same number of steps $\leq K$. 

For simplicity and without loss of generality assume that the Ex-BMDP is in state $x=0$ at timestep $t=0$. Then the state at timestep $k$ is given by:
\begin{equation}
    (k + (q-p)r)\,\%\,q
\end{equation}
where $r$ is the number of times the action $R$ is selected while the Ex-BMDP is in the state 0, and is at most $\lceil k/p\rceil$ (because the agent can start to take the ``short'' p-length loop at most this many times in k total steps.) Now, let $x,x'$ be any two states in $\mathcal{X}\setminus \{0\}$ such that 
\begin{equation}
    (x-x') \equiv (q-p)((q-1)/2) \pmod{q} \label{eq:single_loop_counterexample_eq_0}
\end{equation}

(Concretely, if $((q-p)(q-1)/2)\,\%\,q \neq q-1$, we can set $x= q-1$; $x' = q-1-((q-p)(q-1)/2)\,\%\,q$; otherwise we can set $x= q-2$, $x' =q-1$.)

Then, a path of length $k$ ends at $x$ if and only if it contains $r$ R-actions at state 0, where:
\begin{equation}
    (k + (q-p)r) \equiv x \pmod{q} \label{eq:single_loop_counterexample_eq_1}
\end{equation}
and such a path ends at $x'$ if and only if it contains $r'$ R-actions at state 0, where
\begin{equation}
    (k + (q-p)r') \equiv x' \pmod{q} \label{eq:single_loop_counterexample_eq_2}
\end{equation}
Subtracting these two congruences (Eq. \ref{eq:single_loop_counterexample_eq_1} and \ref{eq:single_loop_counterexample_eq_2}) gives:

\begin{equation}
    (q-p)(r-r') \equiv x-x' \pmod{q}. \label{eq:single_loop_counterexample_eq_3}
\end{equation}

Substituting in Eq. \ref{eq:single_loop_counterexample_eq_0}  gives
\begin{equation}
    (q-p)(r-r') \equiv  (q-p)((q-1)/2)  \pmod{q}.
\end{equation}

Because $q$ is prime and $q-p < q$, the quantities $(q-p)$ and $q$ are relatively prime, so we can safely divide off $(q-p)$:

\begin{equation}
    r-r' \equiv (q-1)/2  \pmod{q}.
\end{equation}

Now either $r-r' > 0$, in which case $r-r'$ is at least $(q-1)/2$, or $r-r' < 0$, in which case $r-r'$ is at most $(q-1)/2 -q = -(q+1)/2$, so $r'-r$ is at least $(q+1)/2$. In either case, because $r$ and $r'$ are non-negative, one of $r$ or $r'$ must be at least  $(q-1)/2$. But $r$ and $r'$ must be at most  $\lceil k/p\rceil$, as noted above. Therefore, for $x$ and $x'$ to both be reached from state 0 in the same number of steps $k$, we must have:

\begin{equation}
   (q-1)/2 \leq \lceil k/p\rceil \leq (k-1)/p +1
\end{equation}
Then 
\begin{equation}
   \frac{1}{2} p(q-3) +1\leq k
\end{equation}
Therefore, AC-State will fail if:
\begin{equation}
  K \leq \frac{1}{2} p(q-3).
\end{equation}
Now, we need to frame this bound in terms of the radius $D$, rather than $p$ and $q$. First, note that the radius of the MDP is in fact $D=q-1$ (In particular, the longest distance between any two states is the distance from $i$ to $i-1$, for $0 < i \leq q-p$.)

Secondly, in order to make this bound as tight as possible in terms of $D$, we would like to make $p$ as close to $D$ as possible. It is known \citep{polymath2014variants} that there are an infinite number of pairs of primes $p,q$ such that  $q-p \leq 246$. Then, in terms of $D$ alone, we have that there are arbitrarily large values of $D$ for which AC-State can fail for all values of $K$ with:
\begin{equation}
  K \leq \frac{1}{2} (D-245)(D-2)  = \frac{1}{2}D^2 + O(D).
\end{equation}

If the Twin Prime conjecture holds, then this bound becomes tighter (although not asymptotically): in this case, there are arbitrarily large values of D for which AC-State can fail for all values of $K$ with:
\begin{equation}
  K \leq \frac{1}{2} (D-1)(D-2)  = \frac{1}{2} D^2 -  \frac{3}{2}D +1
\end{equation}

Finally, because we have shown that such MDPs can be constructed for an infinite number of arbitrarily large values of $D$, we can define $h(D) = \frac{1}{2} (D-245)(D-2) $ for such values of D, and $h(D) =0$ elsewhere. By the definition of $\limsup$, we have:
     \begin{equation}
        \limsup _{D\rightarrow \infty} \frac{h(D)}{\frac{1}{2} D^2} = 1
     \end{equation}
as desired, and thus $h(D) \in \Omega(D^2)$ (by the Hardy-Littlewood definition) as desired.
\end{proof}
While we have shown that we do in fact need to use  $K=\Omega(D^2)$ steps in the AC-State method (in environments with aperiodic latent dynamics, where AC-State works at all), it is natural to ask whether  $\Omega(D^2)$ steps are still necessary when using the ACDF loss. In the ACDF case, it is \textit{not} enough to show that the multistep inverse loss can conflate two individual states (i.e., to show that two states can have a witness distance quadratic in $D$, as above). Rather, in order to fail, the method must conflate some \textit{sets} of states such that the resulting dynamics on the combined ``states'' produced by the encoder are deterministic. In fact, this requirement does not hold for the construction above: empirically, in Section \ref{sec:numerica} we see that ACDF is able to consistently, successfully infer $\phi$ even with \textit{one-step} inverse dynamics ($K=1$) on the version of this MDP with $p=3$, $q=5$.\footnote{Shown in the first numerical experiment at the top of Figure \ref{fig:numeric_results}. Note that while the bound above for this problem would suggest that AC-State will fail for $K\leq 3$, a tighter analysis of the specific instance shows that it will in fact fail for $K \leq 6$, which is confirmed in the experimental results.} 

However, we are still able to construct a family of Ex-BMDPs such that ACDF will fail for $K$ smaller than a function $\in \Omega(D^2)$,  specifically, a function that goes as $D^2/4$.

\begin{proposition}
       $\forall D$, there exists an Ex-BMDP such that a minimal endogenous latent dynamics has a diameter of $D$, and such that the ACDF algorithm given by Equation \ref{eq:acdf} with a uniform exploration policy will return an encoder that does \textbf{not} produce a valid endogenous state representation using $K \leq h(D)$ -step inverse dynamics, where $h(D) \in \Omega(D^2)$. Here, we are using the Hardy-Littlewood $\Omega$ notation (not to be confused with Knuth's $\Omega$), and specifically $h(D)\sim D^2/4$, in the sense that:
     \begin{equation}
        \limsup _{D\rightarrow \infty} \frac{h(D)}{\frac{1}{4} D^2} = 1
     \end{equation}
\end{proposition}
\begin{proof}
Consider two arbitrary primes $p,q$, with $p < q$, and let these primes define a deterministic Ex-BMDP $M$ with transition function $T(x,a)$ as follows:
\begin{itemize}
    \item $2q$ states, labeled $\mathcal{X} = \{0,...,q-2,q-1; 0',1'...,(q-1)'\}$
    \item Two actions, $\{L,R\}$
    \item $T(s,a)$ defined as:
    \begin{itemize}
        \item $T(0,L) = 1$
        \item $T(0,R) = (q-p+1)'$
        \item $T(0',L) = 1'$
        \item $T(0,R) = q-p+1$
        \item $\forall x\in \{1,...,q-1\}$: $T(x,L) = T(x,R) = (x+1)\,\%
        \,q$; and $T(x',L) = T(x',R) = ((x+1)\,\%
        \,q)'$, where $\%$ is the modulo operator.
    \end{itemize}
\end{itemize}
First, we note that this Ex-BMDP is irreducible and aperiodic. To show that it is irreducible, simply note that we can reach any state from any other state (i.e., any state $b$ can be reached from any state $a$ by simply taking action L $(b-a)\,\%\,q$ times; any state $b'$ can be reached from $a'$ in a similar manner; any $b'$ can be reached from $a$ by first going from $a$ to $0$ using L actions as described, using the action R to go from $0$ to $(q-p+1)'$, and then using L actions to go to $b'$; $b$ can be reached from $a'$ similarly). For aperiodicity, note that the state $0 \in \mathcal{X}$ has a self-loop of length $q$ (by taking L and then any other $q-1$ actions) and a self-loop of length $2p$ (by taking $R$ and then any other $p-1$ actions to get to state $0'$, then repeating this process to return to $0$). Because $p$ and $q$  are both prime and $q \neq 2$, the g.c.d. of the lengths of these two self-loops is $1$, so the Markov chain is aperiodic.  Then, by applying Lemma \ref{lemma:one-to-one}, we know that the \textit{only} valid endogenous latent encoder (up to a relabeling permutations) is the trivial one: $\phi(x) = x$. We can thus regard the states $\mathcal{S} = \mathcal{X}$ as the only ``true'' endogenous states: any encoder which maps any two states in $\mathcal{X}$ to the same endogenous state must be incorrect. 

We will show that there is an alternate encoder, $\phi'$, such that the dynamics on the inferred ``states'' produced by $\phi'$ are deterministic and the multistep inverse loss of $\phi'$ is minimal if $K \leq h(D)$, where $h(D) \in \Omega(D^2)$. Specifically, $\forall x \in [0,q-1]$, $\phi'$ maps both $x$ and $x'$ to the same element. To avoid ambiguity, we will refer to this element as $x^*$. Thus, for example, the element $3^*$ is the image of the set $\{3,3'\}$ under $\phi'$. Because $\phi'$ can only output $q$ unique elements, its range is smaller than that of the correct encoder $\phi$, and therefore ACDF will produce $\phi'$ instead of $\phi$ if both have the same loss. 

To show that $\phi'$ has zero forward dynamics loss, we need to verify that the dynamics on the elements $x^*$ are deterministic. In fact, they are:
\begin{itemize}
    \item $T(0^*,L) = 1^*$
    \item $T(0^*,R) = (q-p+1)^*$
    \item $\forall x\in \{1,...,q-1\}$: $T(x^*,L) = T(x^*,R) = ((x+1)\,\%\,q)^*$; 
\end{itemize}
which happen to be the dynamics of the Ex-BMDP used in the proof of Proposition \ref{prop:ACState_fail}.

As in the construction used in the proof of Proposition \ref{prop:ACState_fail}, we note that the first action only affects the dynamics, and hence will only be predictable, if the Ex-BMDP is initially in state $0$ or state $0'$. We also note that by the symmetry of the dynamics:
\begin{equation}
   \forall x\in \{0,...,q-1\},\, k\in \mathbb{N}\,\,\,\,\,\,\Pr(a_t = L | x_t=0,x_{t+k}=x) = \Pr(a_t = L | x_t=0',x_{t+k}=x') 
\end{equation}
and:
\begin{equation}
   \forall x\in \{0,...,q-1\},\, k\in \mathbb{N}\,\,\,\,\,\,\Pr(a_t = L | x_t=0',x_{t+k}=x) = \Pr(a_t = L | x_t=0,x_{t+k}=x') 
\end{equation}
Therefore, up to symmetry, the \textit{only} case in which  distinguishing states $x$ and $x'$ (where  $x\in \{0,...,q-1\}$) is necessary to minimize the multistep-inverse loss is when it is needed to distinguish 
\begin{equation}
  \Pr(a_t = L | x_t=0,x_{t+k}=x') 
\end{equation}
from 
\begin{equation}
  \Pr(a_t = L | x_t=0,x_{t+k}=x) 
\end{equation}
for some \textit{particular} choice of $x$ and $k$: otherwise $\phi'$ will be a sufficient encoder to minimize the loss.

For simplicity and without loss of generality assume that the Ex-BMDP is in state $x=0$ at timestep $t=0$. Then the state at timestep $k$ is given by:
\begin{equation}
    x_k = \begin{cases}
  (k + (q-p)r)\,\%\,q  &\text{if } r \text{ even} \\
   ((k + (q-p)r)\,\%\,q)'  &\text{if } r \text{ odd} \\
\end{cases}
\end{equation}
where $r$ is the number of times the action $R$ is selected while the Ex-BMDP is in the state $0$ \textit{or} the state $0'$, and is at most $\lceil k/p\rceil$. (because the agent can start to take the ``short'' p-length path from $0$ to $0'$ or vice-versa at most this many times in k total steps.) 
Then, in order to reach both $x$ and $x'$ in the same number of steps $k$, we need, for some even $r$  and odd $r'$:
\begin{equation}
    (k + (q-p)r)\,\%\,q  = (k + (q-p)r')\,\%\,q 
\end{equation}
which implies
\begin{equation}
    k + (q-p)r  \equiv k + (q-p)r' \pmod{q}
\end{equation}
Rearranging:
\begin{equation}
   0  \equiv (q-p)(r'-r) \pmod{q}
\end{equation}
Because $q$ is prime and $p<q$, we know that $q-p$ is relatively prime with $q$, so this equivalence implies:
\begin{equation}
   0  \equiv r'-r \pmod{q}
\end{equation}
Because $r'$ is odd and non-negative, this equivalence implies that $r' \geq q$. But then we have:
\begin{equation}
    q \leq r' \leq \lceil k/p\rceil \leq (k-1)/p +1
\end{equation}
Which gives us:
\begin{equation}
    (q-1)p +1 \leq k
\end{equation}
Therefore, ACDF will fail if:
\begin{equation}
  K \leq(q-1)p.
\end{equation}
Now, we need to frame this bound in terms of the radius $D$, rather than $p$ and $q$. First, note that the radius of the MDP is $D=2q-1$ (In particular, the longest distance between any two states is the distance from $1$ to $(q-p)'$, which requires going ``all the way around'' the non-primed states to $0$, going from $0$ to $(q-p+1)'$, and then going ``all the way around'' the primed states to $(q-p+1)'$.) then in terms of $D$ and the ``prime gap'' $g:= q-p$:
\begin{equation}
  K \leq \left( \frac{1}{2} (D-1) \right) \cdot \left( \frac{1}{2} (D+1)  - g\right) 
\end{equation}
As discussed in the proof of Proposition \ref{prop:ACState_fail}, there are an infinite number of pairs of primes with prime gaps bounded by a constant (=246), and under the Twin Primes conjecture, this bound can be made even tighter, with $g=2$. In either case, we have, for arbitrarily large values of $D$:
\begin{equation}
  K \leq  \frac{1}{4}D^2 + O(D).
\end{equation}
One can then proceed exactly as in Proposition \ref{prop:ACState_fail} to prove this proposition.
\end{proof}

\section{ACDF returns a correct, minimal Encoder } \label{sec:ACDF}
We show both that  every function $\phi$ which minimizes the ACDF loss is a valid endogenous state encoder of the Ex-BMDP, and that there exists a valid, minimal-state endogenous state encoder of the Ex-BMDP which minimizes the ACDF loss. This is sufficient to prove that the minimal-state encoder $\phi'$ returned Equation \ref{eq:acdf} will be a valid, minimal-state endogenous state encoder of the Ex-BMDP. Our main result is given as Theorem \ref{thm:main_thm} below.

To show that a particular $\phi'$ is a valid endogenous state encoder, we will use the fact that, by the definition of the Ex-BMDP, there must exist \textit{at least one} minimal endogenous state representation, which we will call $s^* \in \mathcal{S}^*$, and a corresponding exogenous state representation, which we will call $e^* \in \mathcal{E}^*$.  We assume this encoder follows the assumptions discussed in Section \ref{sec:assumptions}. Even though it is arbitrarily chosen, for simplicity in our proofs we will call this the ``canonical'' endogenous/exogenous representation. The ``canonical'' endogenous encoder $\phi^*$, which maps $\mathcal{X}$ to $\mathcal{S}^*$, and ``canonical'' exogenous encoder $\phi_e^*$, which maps $\mathcal{X}$ to $\mathcal{E}^*$, are also defined. We will similarly define the canonical emission distribution  $\mathcal{Q}^*$, such that $ x_{t} \sim \mathcal{Q}^*(x|s,e)$, and $\phi^*$ and $\phi_e^*$ are inverses of $\mathcal{Q}^*$.

We will also introduce a pair of new objects, which we refer to as the  ``enhanced exogenous encoder'' $\bar{\phi_e}^*(x)$ and ``enhanced exogenous state'' $\bar{\mathcal{E}}^*$. These are relevant if the dynamics on the canonical endogenous states $\mathcal{S}^*$ are periodic. Specifically, let $k$ be the periodicity of the canonical endogenous transition function $T^*$, and let $\text{c}(s) \in \mathcal{S}^* \rightarrow [0,k-1]$ refer to the cyclic class of the canonical state $s \in  \mathcal{S}^*$. (We can fix any labeling of these classes; if $T^*$ is aperiodic, then $\forall s, \text{c}(s) = 0$). Then the ``enhanced exogenous state'' is defined by concatenating the canonical exogenous state with $ \text{c}(\phi^*(x))$; that is,  $\bar{\phi_e}^*(x) := (\phi_e^*(x), \text{c}(\phi^*(x)))$. Note that because the cyclic class of $s$ evolves  independently of actions, the dynamics of states in $\bar{\mathcal{E}}^*$ is Markovian and exogenous. For a state $e' \in \bar{\mathcal{E}}^*$, we will refer to the two components as $e'[0]$ and $e'[1]$.

We then introduce the following ``simulation lemma'':

\begin{lemma}
    Consider an Ex-BMDP with states $\mathcal{X}$, transition function $\mathcal{T}$, canonical endogenous states  $\mathcal{S}^*$ and canonical endogenous states  $\mathcal{S}^*$, canonical encoders $\phi^*$ and  $\phi_e^*$, and canonical emission distribution $\mathcal{Q}^*$. If some encoder $\phi'$ exists, such that:
    \begin{enumerate}
        \item $\forall x \in \mathcal{X},$ there is a deterministic function which maps from the pair $(\bar{\phi_e^*}(x),\phi'(x))$  to $\phi^*(x)$.
        \item The dynamics on the encoded latent states produced by $\phi'(x)$ are deterministic. That is, if we fix any $s \in \text{Range}(\phi')$ and $a \in \mathcal{A}$, then $\forall x \in \{x| \phi'(x) = s\},$ let $x' \sim \mathcal{T}(x,a)$, then $\phi'(x')$ takes a single deterministic value, a function only of $s$ and $a$. Additionally, we assume that the range of $\phi'$ is discrete and finite.
    \end{enumerate}
    then $\phi'$ is a valid endogenous latent encoder for the Ex-BMDP.
    \label{lemma:simulation}
\end{lemma}
\begin{proof}
We will show that the Ex-BMDP's transition function $\mathcal{T}$ can be decomposed into endogenous and exogenous parts as defined in Equation \ref{eq:def_exbmdp}, such that the endogenous encoder is $\phi'$, and the exogenous decoder is $\bar{\phi_e}^*(x)$. We will refer to the endogenous/exogenous states under this decomposition as $s'$ and $e'$. We showed above that the enhanced exogenous state has Markovian dynamics which do not depend on actions; we will refer to these dynamics as $\mathcal{T}_e'$.

By assumption, there is a deterministic function $F(\cdot,\cdot)$, which maps $(\phi'(x),\bar{\phi_e^*}(x))$  to $\phi^*(x)$. There is also, by assumption, a deterministic transition function $T'(\cdot, \cdot)$ which maps $\phi'(x)$ and $a$ to $\phi'(x')$, where $x'$ is any element in the support of $\mathcal{T}(x,a)$.

Then, we can define the emission distribution $\mathcal{Q'}$ as:

\begin{equation}
   \forall x \in \mathcal{X}, s' \in \mathcal{S}', e' \in \bar{\mathcal{E}}',\,\, \Pr_{Q'}(x_t = x | s'_t = s', e'_t = e' ) :=  \Pr_{Q^*}(x_t = x | s^*_t = F(s',e'), e^*_t = e'[0] ) \label{eq:def_q'}
\end{equation}

We now show that:
\begin{enumerate}
    \item The complete transition function given by:
    \begin{equation}
\begin{split}
    x_{t+1} &\sim \mathcal{Q}'(x|s'_{t+1},e'_{t+1}),\\
    s'_{t+1} &= T'(s'_{t}, a_{t}), \,\,\,\,\, s'_t = \phi'(x_t)\\
    e'_{t+1} &\sim \mathcal{T}'_e(e'|e'_{t}),\,\,\,\,\, e'_t = \bar{\phi_e}^*(x),\\
\end{split} 
\end{equation}
is equivalent to the Ex-BMDP transition function $\mathcal{T}$, which by assumption can be expressed as:
\begin{equation}
    \begin{split}
    x_{t+1} &\sim \mathcal{Q}^*(x|s^*_{t+1},e^*_{t+1}),\\
    s^*_{t+1} &= T^*(s^*_{t}, a_{t}), \,\,\,\,\, s^*_t = \phi^*(x_t)\\
    e^*_{t+1} &\sim \mathcal{T}^*_e(e^*|e^*_{t}),\,\,\,\,\, e^*_t = \phi_e^*(x),\\
\end{split} 
\end{equation}
To show this fact, is it sufficient to demonstrate that the overall transition probabilities on $\mathcal{X}$ are equal in the two models. That is, we must show that:

\begin{equation}
  \Pr_\mathcal{T}(x_{t+1} \hspace{-3pt}= x|x_t,a_t) = \hspace{-6pt}\sum_{e' \in \bar{\mathcal{E}}^*} \hspace{-3pt}\Pr_{Q'}(x_{t+1}\hspace{-3pt}=x| s_{t+1}' \hspace{-3pt}=T'(\phi'(x_t) ,a_t), e'_{t+1}\hspace{-3pt}= e') \cdot \Pr_{\mathcal{T}_e'}( e' | e'_t = \bar{\phi}^*_e(x_t)) 
\end{equation}.

We can proceed as follows:
\begin{align}
    \sum_{e' \in \bar{\mathcal{E}}^*} \Pr_{Q'}(x_{t+1} =x| s_{t+1}' =T'(\phi'(x_t) ,a_t), e'_{t+1} = e') \cdot \Pr_{\mathcal{T}_e'}( e' | e'_t = \bar{\phi}^*_e(x_t)) &=\\
    \sum_{e' \in \bar{\mathcal{E}}^*} \hspace{-3pt}\Pr_{Q'}(x_{t+1}\hspace{-3pt}=x| s_{t+1}' =\phi'(x_{t+1}), e'_{t+1} = e') \cdot \Pr_{\mathcal{T}_e'}( e' | e'_t = \bar{\phi}^*_e(x_t)) &= \label{eq:main_proof_lemma_1}\\  
    \sum_{e' \in \bar{\mathcal{E}}^*} \hspace{-3pt}\Pr_{Q^*}(x_{t+1}\hspace{-3pt}=x | s^*_{t+1} = F(\phi'(x_{t+1}),  e'), e^*_{t+1} = e'[0] )
    \cdot \Pr_{\mathcal{T}_e'}( e' | e'_t = \bar{\phi}^*_e(x_t)) &=\label{eq:main_proof_lemma_2}\\
    \sum_{e' \in \bar{\mathcal{E}}^*} \hspace{-3pt}\Pr_{Q^*}(x_{t+1}\hspace{-3pt}=x | s^*_{t+1} \hspace{-3pt}= \hspace{-3pt}F(\phi'(x_{t+1}),  e'), e^*_{t+1} \hspace{-3pt}= e'[0], \text{c}(s^*_{t+1}) \hspace{-3pt}= \hspace{-3pt}e'[1]  ) \hspace{-2pt} \cdot \hspace{-2pt} \Pr_{\mathcal{T}_e'}( e' | e'_t = \bar{\phi}^*_e(x_t)) &=\label{eq:main_proof_lemma_3}\\  
    \sum_{e' \in \bar{\mathcal{E}}^*} \hspace{-3pt}\Pr_{Q^*}(x_{t+1}\hspace{-3pt}=\hspace{-2pt}x | s^*_{t+1} \hspace{-3pt}= \hspace{-3pt}F(\phi'(x_{t+1}),  e'_{t+1}), e^*_{t+1} \hspace{-3pt}= \hspace{-2pt}e'[0], \text{c}(s^*_{t+1}) \hspace{-3pt}= \hspace{-3pt}e'[1]  )
   \hspace{-2pt} \cdot \hspace{-2pt} \Pr_{\mathcal{T}_e'}( e' | e'_t \hspace{-2pt}= \hspace{-2pt}\bar{\phi}^*_e(x_t))\hspace{-3pt}&=\label{eq:main_proof_lemma_4}\\   
    \sum_{e' \in \bar{\mathcal{E}}^*} \hspace{-3pt}\Pr_{Q^*}(x_{t+1}\hspace{-3pt}=x | s^*_{t+1} = \phi^*(x_{t+1}), e^*_{t+1} = e'[0], \text{c}(s^*_{t+1})= e'[1]  )  \cdot \Pr_{\mathcal{T}_e'}( e' | e'_t = \bar{\phi}^*_e(x_t)) &=\label{eq:main_proof_lemma_5}\\     
    \sum_{e' \in \bar{\mathcal{E}}^*} \hspace{-3pt}\Pr_{Q^*}(x_{t+1}\hspace{-3pt}=x | s^*_{t+1} =T^*(\phi^*(x_t),a_t), e^*_{t+1} = e'[0] )  \cdot \Pr_{\mathcal{T}_e'}( e' | e'_t = \bar{\phi}^*_e(x_t)) &=\label{eq:main_proof_lemma_6}\\ 
    \begin{split}
            \sum_{e' \in \bar{\mathcal{E}}^*} \hspace{-3pt}\Pr_{Q^*}(x_{t+1}\hspace{-3pt}=x | s^*_{t+1} =T^*(\phi^*(x_t),a_t), e^*_{t+1} = e'[0] ) &\\
            \cdot \Pr_{\mathcal{T}_e^*}(e'[0]  | e^*_t = \phi^*_e(x_t)) \cdot  \mathbbm{1}_{e'[1] \equiv \text{c}(\phi^*(x_t)) +1\hspace{-6pt} \pmod{k}} &= \label{eq:main_proof_lemma_7}
    \end{split}\\
    \sum_{e'[0] \in \mathcal{E}^*}\hspace{-3pt}\Pr_{Q^*}(x_{t+1}\hspace{-3pt}=x | s^*_{t+1} =T^*(\phi^*(x_t),a_t), e^*_{t+1} = e'[0] )  \cdot \Pr_{\mathcal{T}_e^*}(e'[0]  | e^*_t = \phi^*_e(x_t)) &=\label{eq:main_proof_lemma_8}\\  
    \sum_{e^* \in \mathcal{E}^*}\hspace{-3pt}\Pr_{Q^*}(x_{t+1}\hspace{-3pt}=x | s^*_{t+1} =T^*(\phi^*(x_t),a_t), e^*_{t+1} = e^* )  \cdot\Pr_{\mathcal{T}_e^*}(e^*  | e^*_t = \phi^*_e(x_t)) &=\label{eq:main_proof_lemma_9}\\   
     \Pr_\mathcal{T}(x_{t+1} \hspace{-3pt}= x|x_t,a_t) \label{eq:main_proof_lemma_10}
\end{align}
where Eq. \ref{eq:main_proof_lemma_1} follows by the assumption that $T'$ exists and is deterministic (Assumption 2 of the Lemma); Eq.  \ref{eq:main_proof_lemma_2}  follows by Assumption 1 of the Lemma;  Eq.  \ref{eq:main_proof_lemma_3} follows from the fact that, if $\Pr_{\mathcal{T}_e'}( e' | e'_t = \bar{\phi}^*_e(x_t))$ is nonzero, then $e'[1]$ must follow $e'_t[1] \pmod{k}$, therefore  $e'[1] = \text{c}(s^*_{t+1})$ for all of these terms; the conditioning is merely making this constraint on $e'$ explicit. Then in Eq. \ref{eq:main_proof_lemma_4} we use the two conditions that  $e'[0] = e^*_{t+1}$  and  $e'[1] = \text{c}(s^*_{t+1})$ to conclude that $e' = e'_{t+1}$, and substitute using this identity in the other condition. Eq. \ref{eq:main_proof_lemma_5} follows from the definition of $F(\cdot,\cdot)$; and Eq. \ref{eq:main_proof_lemma_6} follows from the definition of the canonical endogenous transition function $T^*$ (we also hide the implicit conditioning/constraint on $e'[1]$, which as discussed above is redundant on all nonzero terms.)  Eq. \ref{eq:main_proof_lemma_7} factors the transition function $\mathcal{T}_e'$ into the canonical exogenous transition function $\mathcal{T}_e^*$ and the deterministic procession of the cyclic class of $s^*$.  Note that $e'[1]$ now only appears in the final indicator function term. Furthermore, the indicator is equal to 1 for exactly one of the possible values of $e'[1]$. Therefore, in Eq. \ref{eq:main_proof_lemma_8} we  split the sum over $e' \in \bar{\mathcal{E}}^*$ into a sum over $e'[0] \in\mathcal{E}^*$ and a sum over $e'[1] \in [0,k-1]$. The latter sum can be factored to apply only to the indicator function term, and the resulting sum is equal to exactly 1, so it can be eliminated. Eq. \ref{eq:main_proof_lemma_9} is simply a notational change, from $e'[0]$ to $e^*$. Finally, Eq. \ref{eq:main_proof_lemma_10} holds directly from the definition of the Ex-BMDP in terms of the canonical endogenous and exogenous dynamics.
\item The conditions on Eq. \ref{eq:def_exbmdp} are met:
\begin{itemize}
    \item $\mathcal{S}'$, the range of $\phi'$ is  finite, by assumption.
    \item The transition function $T'$ is deterministic, by assumption.
    \item The transition function $\mathcal{T}_e'$ is Markovian and independent of actions, because it consists of factor $\mathcal{T}_e^*$, which is Markovian and independent of actions, and a cyclic term which evolves deterministically as $n \rightarrow (n +1)\,\%\,k $, and is thus also Markovian and independent of actions.
    \item $Q'$ follows the ``block'' assumption, because the encoders $\phi'$ and $\bar{\phi}_e^*$ are both functions, which implies that every value of $x \in \mathcal{X}$ corresponds to a unique $(s',e')$ pair.
    \item  $\phi'$ and $\bar{\phi}_e^*$ are inverses of $Q'$:
    in other words, if $x\sim Q'(s',e')$, then $\phi'(x) = s'$ and  $\bar{\phi}_e^*(x) = e'$. This statement is only meaningful (i,e., $ Q'(s',e')$ is only defined) in cases where $s'$ and $e'$ can co-occur (which is not always the case if $\mathcal{T}_e'$ and $T'$ are both periodic with the same period.)  The sets $\mathcal{S}'$ and $\bar{\mathcal{E}}^*$ are defined as the ranges of their respective encoders, so $(s',e')$ can co-occur if and only if there is some  $x' \in \mathcal{X}$ such that $s' = \phi'(x')$ and $e' = \bar{\phi}_e^*(x')$ . Then we must show that  $x\sim Q'(\phi(x'),\bar{\phi}_e^*(x'))$ implies $\phi(x) = \phi(x')$ and $ \bar{\phi}_e^*(x) = \bar{\phi}_e^*(x')$.  For the first implication, note that, by Eq. \ref{eq:def_q'}, we have that $x\sim Q^*(F(\phi(x'), \bar{\phi}_e^*(x')), \bar{\phi}_e^*(x')[0])$, and therefore that $\phi^*(x) = F(\phi(x'),\bar{\phi}_e^*(x'))$ and $\phi_e^*(x) =  \bar{\phi}_e^*(x')[0] = \phi_e^*(x')$ . Then, by the definition of $F$,  $\phi^*(x) = \phi^*(x')$ and $\phi_e^*(x) = \phi_e^*(x')$. Then, from the dynamics of the canonical latent encoding (specifically, the bounded diameter assumption of $\mathcal{S}^*$ and lack of transient states in $\mathcal{E}^*$), we know that there must exist some $x''$ such that $x''$ can transition to both $x$ and $x'$ under the same action $a$. Then, by the determinism of $T'$, we have that $T'(\phi'(x''),a) = \phi'(x) = \phi'(x')$, so $\phi(x) = \phi(x')$ as desired. Given this result and the fact that $\phi_e^*(x) = \phi_e^*(x')$, it follows by definition that $ \bar{\phi}_e^*(x) = \bar{\phi}_e^*(x')$. 
\end{itemize}
\end{enumerate}

\end{proof}
Given this lemma, we can now prove our main theorem:
\begin{theorem}
    Given an Ex-BMDP, under the assumptions listed in Section \ref{sec:assumptions}, there exists a correct minimal-state endogenous state encoder $\phi^*$ that will minimize the ACDF loss given in Equation \ref{eq:acdf} on the Ex-BMDP. Conversely, any encoder $\phi'$ which minimizes the ACDF loss given in Equation \ref{eq:acdf} on the Ex-BMDP is a correct endogenous state representation for the Ex-BMDP. As a consequence, the encoder $\phi'$ which minimizes Equation \ref{eq:acdf} with the minimum number of output states is a correct minimal-state endogenous representation. \label{thm:main_thm}
\end{theorem}
\begin{proof}

Our general approach is as follows:
\begin{enumerate}
    \item We show that there exists a minimal-state endogenous encoder which \textit{simultaneously} minimizes \textit{both} the multi-step inverse loss term and the latent forward loss term in Equation \ref{eq:acdf}. As a consequence, the encoder also minimizes the overall loss, the sum of these two terms.
    \item We show that any $\phi'$ which achieves the minimum possible value of the \textit{both} terms in the ACDF loss (\textit{independently}) is a correct endogenous-state encoder.
    \item Due to item (1) above, we know that it is possible to simultaneously minimize both loss terms. Therefore, \textit{any} $\phi'$ which achieves the minimum value of the overall loss in  Equation \ref{eq:acdf} must also minimize both loss terms independently. Then we can conclude that any $\phi'$ which minimizes the overall loss $\mathcal{L}_{\text{ACDF}}$ is a correct endogenous latent state encoder. 
    \item Finally, because all minimizers of $\mathcal{L}_{\text{ACDF}}$ are correct endogenous state encoders, and we know that at least one of these minimizers is minimal-state among correct encoders, it follows that the  minimizers of $\mathcal{L}_{\text{ACDF}}$ which have the fewest numbers of output states are all minimal-state endogenous state encoders.
\end{enumerate}
It show that points (1) and (2) above are correct; points (3) and (4) follow immediately.

\textbf{(1) A correct, minimal-state encoder minimizes both loss terms.}

Let $\phi^*$ be the correct ``canonical'' minimal-state endogenous state encoder, as described in Section~\ref{sec:assumptions}. To show that $\phi^*$ minimizes the multistep inverse loss, we can proceed exactly as in Proposition 5.1 in \cite{lamb2022guaranteed}.

However, for completeness, we present a version of the proof here. Recall that, by assumption (4) listed in Section \ref{sec:assumptions}, the behavior policy used to collect data depends only on the endogenous state $s^*$. Let $M$ be the represent the transition operator of the entire Ex-BMDP on $\mathcal{X}$ under the behavior policy, so that $x_{t+k} \sim M^k x_t$. Similarly, let $M_e$ represent the exogenous-state transition operator under the exogenous transition function $\mathcal{T}_e^*$ and  $M_s$ represent the endogenous-state transition operator under the behavioral policy $\pi$ and the exogenous transition function $T^*$. Then, by the definition of the Ex-BMDP dynamics, we have that

\begin{equation}
\Pr_{M^k}(x_{t+k} = x | x_t) = \Pr_{Q^*}(x | \phi^*(x), \phi^*_e(x)) \cdot \Pr_{M_e^k}(e^*_{t+k} = \phi^*_e(x) | \phi^*_e(x_t))  \cdot \Pr_{M_s^k}(s^*_{t+k} = \phi^*(x) | \phi^*(x_t)) \label{eq:full_mk}
\end{equation}

Additionally:

\begin{equation}
\begin{split}
    &\Pr_{\mathcal{M}^{k-1}, \mathcal{T}} (x_{t+k} = x | a_t = a, x_t) =\\ &\Pr_{Q^*}(x | \phi^*(x), \phi^*_e(x)) \cdot \Pr_{M_e^k}(e^*_{t+k} = \phi^*_e(x) | \phi^*_e(x_t))  \cdot \Pr_{M_s^{k-1}}(s^*_{t+k} = \phi^*(x) | s^*_{t+1} = T^*(\phi^*(x_t))) \label{eq:partial_mk}
\end{split}
\end{equation}

Now, for any $(x_t, x_{t+k};k)$ tuple, the multistep-inverse loss in Equation \ref{eq:acdf} is minimized by the function $f^{\text{opt.}}$ defined as: $f^{\text{opt.}}_{a_t}(x_t,  x_{t+k}; k) := \Pr_{\mathcal{D}_{(k)}}(a_t|x_t,  x_{t+k}) $ where $\mathcal{D}_{(k)}$ is the sampling distribution over which we compute the loss, as defined in Section \ref{sec:assumptions}; note that the transition function on this distribution is the $M$ defined above. Now, we can write:

\begin{equation}
\begin{split}
 \mathop{\Pr}_{\mathcal{D}_{(k)}}(a_t= a|x_t,  x_{t+k} = x)  &=\\ 
 \frac{
 \mathop{\Pr}_{\pi}(a_t= a| x_t)
 \cdot \mathop{\Pr}_{\mathcal{M}^{k-1}, \mathcal{T}}(x_{t+k}=x |a_t= a, x_t) }{\mathop{\Pr}_{M^k}(x_{t+k}=x| x_t)} &=\\
 \frac{
\begin{multlined}
     \mathop{\Pr}_{\pi}(a_t= a| \phi^*(x_t))
 \cdot  \\ \mathop{\Pr}_{Q^*}(x | \phi^*(x), \phi^*_e(x)) \cdot \mathop{\Pr}_{M_e^k}(e^*_{t+k} = \phi^*_e(x) | \phi^*_e(x_t))  \cdot \mathop{\Pr}_{M_s^{k-1}}(s^*_{t+k} = \phi^*(x) | s^*_{t+1} = T^*(\phi^*(x_t)))
 \end{multlined}}{\mathop{\Pr}_{Q^*}(x | \phi^*(x), \phi^*_e(x)) \cdot \mathop{\Pr}_{M_e^k}(e^*_{t+k} = \phi^*_e(x) | \phi^*_e(x_t))  \cdot \mathop{\Pr}_{M_s^k}(s^*_{t+k} = \phi^*(x) | \phi^*(x_t)) } &=\\
  \frac{
 \mathop{\Pr}_{\pi}(a_t= a|\phi^*(x_t))
 \cdot  \mathop{\Pr}_{M_s^{k-1}}(s^*_{t+k} = \phi^*(x) | s^*_{t+1} = T^*(\phi^*(x_t))) }{\mathop{\Pr}_{M_s^k}(s^*_{t+k} = \phi^*(x) | \phi^*(x_t)) } &\\
\end{split}
\end{equation}

Where the we use Equations \ref{eq:full_mk} and \ref{eq:partial_mk}, and the fact that $\pi$ depends only on the endogenous state, in the third line. Note that the final line of the equation depends on $x$ and $x_t$ only through $\phi^*(x)$ and $\phi^*(x_t)$ respectively, so $f^{\text{opt.}}$ can be equivalently written as a function only of $\phi^*(x_{t+k})$, $\phi^*(x_t)$, and $k$. Then the encoder $\phi^*$ is able to achieve the minimum possible value of the multistep inverse loss term.

We now show that $\phi^*$ minimizes the latent forward loss as well. Because $T^*$ is deterministic, it follows that, for all transitions $(x_t,a_t,x_{t+1})$, $\phi^*(x_{t+1})$ is a deterministic function of $\phi^*(x_t)$ and $a_t$ (in particular, $\phi^*(x_{t+1}) = T^*(\phi^*(x_t),a_t)$). Then the forward-prediction loss term will be exactly zero, achieved by setting $g_{s'}(s,a) :=  \mathbbm{1}_{s' = T^*(s,a)}$. This is the minimum possible value for a negative-$\log$ loss term. 

\textbf{(2) An encoder $\phi'$ which minimizes both loss terms is a correct endogenous encoder.}

    By assumption, the Ex-BMDP must have some canonical endogenous states  $\mathcal{S}^*$ and canonical exogenous states  $\mathcal{E}^*$, canonical encoders $\phi^*$ and  $\phi_e^*$, and canonical emission distribution $\mathcal{Q}^*$. In order to use the above Lemma \ref{lemma:simulation}, we will first show that if $\phi'$ minimizes the multistep-inverse loss term, then there is a deterministic mapping from $(\bar{\phi_e^*}(x),\phi'(x))$  to $\phi^*(x)$.
    
    We proceed by contradiction. Suppose, for two distinct $x, x'$, we have that $(\bar{\phi_e^*}(x),\phi'(x)) = (\bar{\phi_e^*}(x'),\phi'(x'))$, but  $\phi^*(x) \neq \phi^*(x')$.  Because $\bar{\phi_e^*}(x)[1] = \bar{\phi_e^*}(x')[1]$, we know that  $\phi^*(x)$ and $\phi^*(x')$ must belong to the same cyclic class in $\mathcal{S}^*$; additionally, we have that  $\phi_e^*(x)=\phi_e^*(x')$. Let $w = W(\phi^*(x),\phi^*(x'))$. Because  $\phi^*(x)$ and $\phi^*(x')$ belong to the same cyclic class, we know by Theorem \ref{thm:finite_wd} that $w \leq D' < \infty$. Let $c \in \mathcal{S^*}$ be the corresponding witness state. By the lack of transient states in $\mathcal{E}^*$, we know that there is some $e^* \in E^*$ such that under the exogenous dynamics $\mathcal{T}_e^*$, $\phi_e^*(x)$ can be reached from $e^*$ in exactly $w$ steps. We can step from $(c,e^*)$ to $(\phi^*(x), \phi_e^*(x))$ in exactly $w$ steps, or from $(c,e^*)$ to $(\phi^*(x'), \phi_e^*(x))$ in exactly $w$ steps.
    By Lemma \ref{lemma:reachability}, there exists an observation $x_c \in \mathcal{X}$ such that $\phi^*(x_c) = c$ and $\phi_e^*(x_c) = e^*$. Then, by Assumption 5 listed in Section \ref{sec:assumptions}, we sample $(x_t = x_c,a_t \in \mathcal{A},x_{t+w} = x)$ with finite probability and $(x_t = x_c,a_t \in \mathcal{A},x_{t+w} = x')$ with finite probability. By assumption, $\phi'(x) = \phi'(x')$, so
    
    \begin{equation}
        f(\phi'(x_c) ,\phi'(x) ; w) = f(\phi'(x_c) ,\phi'(x') ; w). \label{eq:same_pred}
    \end{equation}
     Note that the multistep inverse loss term is minimized on \textit{every} tuple $(x_t, x_{t+k}; k)$ by the trivial encoder $\phi_{\text{triv}}$ which maps $\phi_{\text{triv}}(x): = x$. Under this encoder, we can simply define $f^{\text{opt.}}_{a_t}(x_t,  x_{t+k}; k) := \Pr_{\mathcal{D}_{(k)}}(a_t|x_t,  x_{t+k}) $ where $\mathcal{D}_{(k)}$ is the sampling distribution as defined in Section \ref{sec:assumptions}. 
     Furthermore, for a fixed  $(x_t, x_{t+k}; k)$ which occurs with nonzero probability, this minimizer  $f^{\text{opt.}}(x_t, x_{t+k}; k)$ is unique: the excess loss conditioned on  $(x_t, x_{t+k}; k)$ is exactly the KL-divergence between $\Pr_{\mathcal{D}_{(k)}}(a_t|x_t,  x_{t+k}) $ and the  distribution output by the classifier $f(x_t, x_{t+k}; k)$. 

    Then, $\phi'$ can only minimize the loss term if, for some $f$ and all $a \in \mathcal{A}$:

    \begin{equation}
    \begin{split}
                \Pr_{\mathcal{D}_{(w)}}(a_t =a|x_t = x_c,  x_{t+w} = x) = &f_a(\phi'(x_c) ,\phi'(x) ; w) = \\
                f_a(\phi'(x_c) ,\phi'(x') ; w) = &  \Pr_{\mathcal{D}_{(w)}}(a_t =a|x_t = x_c,  x_{t+k} = x'). \label{eq:thm_contradiction_1}
    \end{split}
    \end{equation}
     Where the middle equality is Equation \ref{eq:same_pred}.

     However, $\Pr_{\mathcal{D}_{(w)}}(a_t =a|x_t = x_c,  x_{t+w} = x) $ and $\Pr_{\mathcal{D}_{(w)}}(a_t =a|x_t = x_c,  x_{t+w} = x') $ have \textit{disjoint support} over values of $a\in \mathcal{A}$. To see why, suppose that some $a'\in \mathcal{A}$ existed such that both of these probabilities were nonzero. This would imply that, for the canonical endogenous latent state $d\in \mathcal{S}^*$ defined as $d:=T^*(c,a')$, it is possible to reach both latent states $\phi^*(x)$ and  $\phi^*(x')$ from $d$ in exactly $w-1$ steps of $T^*$. But then $W(\phi^*(x),\phi^*(x')) \leq w-1$, which contradicts the definition of $w:= W(\phi^*(x),\phi^*(x'))$.
     
     Then we can conclude that distributions $\Pr_{\mathcal{D}_{(w)}}(a_t =a|x_t = x_c,  x_{t+w} = x) $ and $\Pr_{\mathcal{D}_{(w)}}(a_t =a|x_t = x_c,  x_{t+w} = x') $ have disjoint support. This leads to a contradiction with Equation \ref{eq:thm_contradiction_1}, which requires the two distributions to be equal. Therefore $\phi'$ cannot minimize the multistep inverse loss term if, for any two distinct $x, x'$, we have  $(\bar{\phi_e^*}(x),\phi'(x)) = (\bar{\phi_e^*}(x'),\phi'(x'))$, but  $\phi^*(x) \neq \phi^*(x')$. Therefore for any  $\phi'$ which minimizes the loss term, there must be a deterministic mapping from $(\bar{\phi_e^*}(x),\phi'(x))$  to $\phi^*(x)$.

     Secondly, we note that any $\phi'$ which minimizes the forward dynamics loss term must produce a set of states $\mathcal{S}'$ with a deterministic transition function. To see this, note that this loss term can be zero with \textit{some} encoder (such as the encoder which maps all $x$ to a single latent state), and further note that the loss is \textit{only} zero if $\phi'(x_{t+1})$ is exactly predictable from  $\phi'(x_{t})$ and $a_t$; which means that the dynamics are deterministic.

     From these two conclusions, that there is a deterministic mapping from $(\bar{\phi_e^*}(x),\phi'(x))$  to $\phi^*(x)$, and that there are deterministic dynamics on the endogenous states produced by $\phi'$, we can conclude by Lemma \ref{lemma:simulation} that $\phi'$ is a valid endogenous encoder of the Ex-BMDP.

\end{proof}

\section{Alternative ``fixes'' to AC-State that do not work} \label{sec:alt_fixes}
While ACDF has nice theoretical properties, one concern is that, unlike the purely \textit{supervised} action-prediction task of AC-State, the ACDF method's latent-forward prediction task has \textit{moving targets}. In other words, the latent code of the next state $s_t+1$ will evolve throughout training, complicating the training optimization. Given this difficulty, it would be desirable to come up with a method for provably learning the endogenous encoder of an Ex-BMDP that does \textit{not} rely on moving targets. We considered two initially promising possibilities, but found that both can fail to capture a correct endogenous state representation.  
\subsection{Full Multi-Step Inverse}
\subsubsection{What is it?/Why was it Promising?}
The multistep inverse method proposed by \cite{lamb2022guaranteed} uses $\phi(x_t)$ and $\phi(x_{t+k})$ to predict $a_t$, that is, the action immediately following $x_t$. A natural extension of this method would be to model the probability distribution over \textit{all} of the actions on the path between  $\phi(x_t)$ and $\phi(x_{t+k})$, that is, $a_t, a_{t+1}.,,,,a_{t+k-1}$, based on $\phi(x_t)$ and $\phi(x_{t+k})$ alone. Equivalently, this modeling can be done autoregressively by, for each $k' \in [0,k-1]$, predicting $a_{t+k'}$ given  $\phi(x_t)$, $\phi(x_{t+k})$, and $a_t, a_{t+1}.,,,,a_{t+k'-1}$. This formulation yields the following modified loss function:
\begin{equation}
\begin{split}
    \mathcal{L}&_\text{AC-State-Full-Multi}(\phi_\theta) :=\\& \min_f \mathop\mathbb{E}_{k \sim \{1,...,K\} } \mathop\mathbb{E}_{k' \sim \{1,...,k-1\} } \mathop\mathbb{E}_{(x_t,a_t,...,a_{t+k'-1}, x_{t+k})}  -\log(  f_{a_{t+k'}}(\phi_\theta(x_t) ,\phi_\theta(x_{t+k}), a_t,...,a_{t+k'-1}; k))
\end{split}
\end{equation}
Conceptually, this technique might seem promising  because the task of predicting  $a_{t+k'}$ from  $\phi(x_t)$, $\phi(x_{t+k})$, and $a_t, a_{t+1}.,,,,a_{t+k'-1}$ can be accomplished by decomposing the problem into predicting $\phi(x_{t+k'})$ from  $\phi(x_t)$ and $a_t, a_{t+1}.,,,,a_{t+k'-1}$, and then predicting $a_{t+k'}$ from $\phi(x_{t+k'})$ and $\phi(x_{t+k})$. In other words, it can be accomplished by composing a \textit{latent forward model} with a \textit{multistep inverse model}. Thus, it would seem to require learning a representation similarly rich to the representation learned by ACDF, without dealing with the moving-target issue caused by explicitly learning a forward dynamics model.

Furthermore, the ``AC-State-Full-Multi'' method successfully learns the encoder for the 5-state periodic Ex-BMDP shown in Figure \ref{fig:witness_dist_figure}-E in the main text, which AC-State fails on. In particular, the encoder must be able to distinguish states $b$ and $c$,  because $\Pr(a_{t+1} = L| s_t=b, s_{t+3} = b,a_t = L) = 0.5$, while $\Pr(a_{t+1} = L| s_t=c, s_{t+3} = c,a_t = L) = 1$.
States $d$ and $e$ can be distinguished similarly.

\subsubsection{Why it fails} \label{sec:full_multi_fail}
Unfortunately, while the ``AC-State-Full-Multi'' prediction task \text{can be accomplished by} learning a deterministic latent forward model and a first-action multistep-inverse model, it does not \textit{require} learning these two things. We show a counterexample here. Consider the following Ex-BMDP on the states $\mathcal{X} = \{a,b,c,a',b',c'\}$, with actions $\mathcal{A} = \{A,B,C\}$ and deterministic transition function $T(x,a)$ shown in Figure \ref{fig:counterexample_full_multi}.

\begin{figure}[h!]
    \centering
\begin{tikzpicture}[node distance=2cm and 4cm]
\node[draw=black,shape=circle](a){a~};
\node[draw=black,shape=circle](b)[below = of a] {b~};
\node[draw=black,shape=circle](c)[below = of b] {c~};
\node[draw=black,shape=circle](a')[right = of a] {a'};
\node[draw=black,shape=circle](b')[below = of a'] {b'};
\node[draw=black,shape=circle](c')[below = of b'] {c'};
\path  (a) edge [bend right]["A",swap, pos=0.9](a');
\path  (a) edge [bend right]["B",swap, pos=0.9](b');
\path  (a) edge [bend right]["C",swap, pos=0.9](c');
\path  (b) edge [bend right]["A",swap, pos=0.9](a');
\path  (b) edge [bend right]["B",swap, pos=0.9](b');
\path  (b) edge [bend right]["C",swap, pos=0.9](c');
\path  (c) edge [bend right]["A",swap, pos=0.9](a');
\path  (c) edge [bend right]["B",swap, pos=0.9](b');
\path  (c) edge [bend right]["C",swap, pos=0.9](c');
\path  (a') edge [bend right]["A",swap, pos=0.9](a);
\path  (a') edge [bend right]["B",swap, pos=0.9](b);
\path  (a') edge [bend right]["C",swap, pos=0.9](c);
\path  (b') edge [bend right]["A",swap, pos=0.9](a);
\path  (b') edge [bend right]["B",swap, pos=0.9](b);
\path  (b') edge [bend right]["C",swap, pos=0.9](c);
\path  (c') edge [bend right]["A",swap, pos=0.9](a);
\path  (c') edge [bend right]["B",swap, pos=0.9](c);
\path  (c') edge [bend right]["C",swap, pos=0.9](b);
\end{tikzpicture}
\caption{Transitions of $T(x,a)$.}
    \label{fig:counterexample_full_multi}
\end{figure}
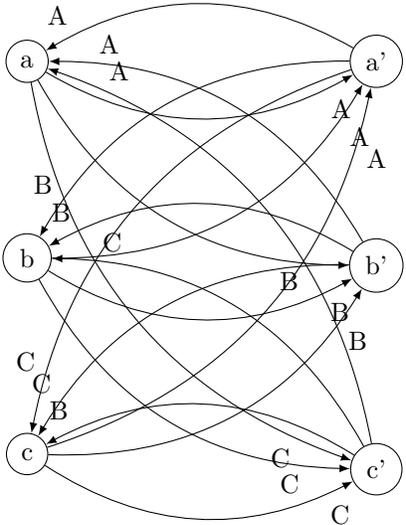
The transition function is defined such that, letting x/X and y/Y represent a/A, b/B, or c/C, we have that for \textit{most} inputs, $T(x,Y) = y'$ and $T(x',Y) = y$, with the exceptions that $T(c',B) = c$ and $T(c',C)= b$.
This state representation turns out to be minimal (the only correct  encoder up to permutation is the trivial one $\phi(x) = x$). To confirm this fact, briefly, note being in $\{a,a'\}$, $\{b,b'\}$, or $\{c,c'\}$ is clearly controllable, and furthermore we must always keep track of whether we are in a primed or non-primed state to determine the effect of actions B and C when we reach $\{c,c'\}$.

However, the ``AC-State-Full-Multi'' loss will be minimized by an encoder $\phi'$ that produces only 5 latent states, with $a$ and $a'$ conflated into a single latent state. The resulting endogenous dynamics are not deterministic, because starting at the state $\{a,a'\}$ and taking the action $B$ may lead to either the state $b$ or state $b'$, and similarly for for the action $C$. (See Figure \ref{fig:counterexample_full_multi_2}).

\begin{figure}[h!]
    \centering
\begin{tikzpicture}[node distance=2cm and 2cm]
\node[draw=black,shape=circle](a){$\{a,a'\}$};
\node[draw=black,shape=circle](b)[below left = of a] {b~};
\node[draw=black,shape=circle](c)[below = of b] {c~};
\node[draw=black,shape=circle](b')[below right = of a] {b'};
\node[draw=black,shape=circle](c')[below = of b'] {c'};
\path  (a) edge [loop above]["A",swap, pos=0.9](a);
\path  (a) edge [bend right]["B",swap, pos=0.9](b');
\path  (a) edge [bend right]["C",swap, pos=0.9](c');
\path  (b) edge [bend right]["A",swap, pos=0.9](a);
\path  (b) edge [bend right]["B",swap, pos=0.9](b');
\path  (b) edge [bend right]["C",swap, pos=0.9](c');
\path  (c) edge [bend right]["A",swap, pos=0.9](a);
\path  (c) edge [bend right]["B",swap, pos=0.9](b');
\path  (c) edge [bend right]["C",swap, pos=0.9](c');
\path  (a) edge [bend right]["B",swap, pos=0.9](b);
\path  (a) edge [bend right]["C",swap, pos=0.9](c);
\path  (b') edge [bend right]["A",swap, pos=0.9](a);
\path  (b') edge [bend right]["B",swap, pos=0.9](b);
\path  (b') edge [bend right]["C",swap, pos=0.9](c);
\path  (c') edge [bend right]["A",swap, pos=0.9](a);
\path  (c') edge [bend right]["B",swap, pos=0.9](c);
\path  (c') edge [bend right]["C",swap, pos=0.9](b);
\end{tikzpicture}
\caption{Nondeterministic transitions on $T(\phi'(x),a)$ for the example in Section \ref{sec:full_multi_fail}.}
    \label{fig:counterexample_full_multi_2}
\end{figure}
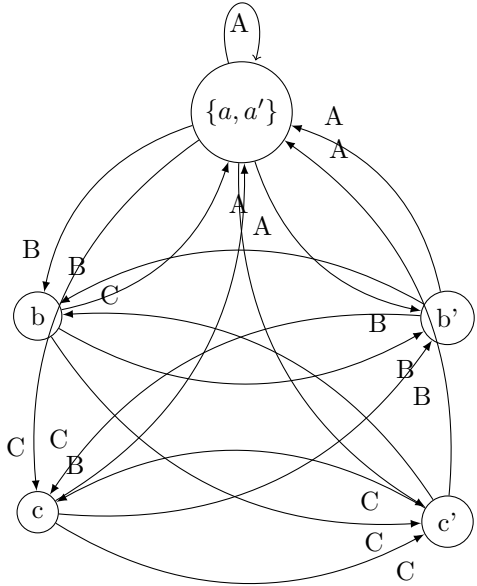

To show that this incorrect $\phi'$ minimizes the ``AC-State-Full-Multi' loss, note that, under a uniform policy:
\begin{itemize}
    \item For any sequence of actions that ends in state $a$ or $a'$, none of the actions can be meaningfully inferred except the \textit{final} action of the sequence, which we know is $A$ with probability 1. Therefore distinguishing whether $x_{t+k} = a$ or $x_{t+k} = a'$ is not necessary to minimize the loss.
    \item For any sequence of actions that begins with state $a$ or $a'$, but ends in b,b',c or c', one can distinguish if the initial state is a or a' by simply looking at the final state and the parity of the length $k$ of the sequence between them (e.g, if the final state is $b$, and $k$ is odd, we know the initial state was a' rather than a). Then it is not necessary for the encoder to distinguish between $x_t = a$ and $x_t = a'$.
\end{itemize}
Therefore, $\phi'$ is sufficient to minimise the ``AC-State-Full-Multi' loss, and is more minimal than the correct endogenous state representation, but is an incorrect endogenous representation. Therefore the ``AC-State-Full-Multi'' loss fails to discover the true minimal endogenous state representation.
\subsection{Artificial Self-Edges/Imprecise $k$} \label{sec:self_edges}
\subsubsection{What is it?/Why was it Promising?}
Note that for any bounded-diameter (i.e., irreducible) MDP, adding any self-edges, even a single self-edge, will make the MDP aperiodic. (A self-edge is a transition from a state to itself.) Making this change to the endogenous dynamics of an Ex-BMDP would therefore eliminate the need to use a latent forward model, as in ACDF. In fact, if we add a self-edge to any state, the witness distance between any pair of states will automatically become $\leq D$: if $c$ is the state with the self-edge, then any states $a$ and $b$ can both be reached in exactly $\max(d(c,a), d(c,b))$ steps. Concretely, if $a$ is the further state, then we can reach $b$ in $d(c,a)$ steps by taking the self-edge at $c$ for $d(c,a) - d(c,b)$  timesteps before going to $b$.

Unfortunately, we cannot simply \textit{alter} the underlying dynamics by adding a self-edge, or many self-edges. However, we might hope to \textit{simulate} such self-edges by randomly duplicating some observations $x$ in the replay buffer, and inserting a ``new'' action symbol in between the duplicated observations. After learning the state abstraction and dynamics, the self-edges with the ``new'' action symbol can simply be removed.

A nearly-equivalent idea is to, rather than trying to predict $a_t$ given $\phi(x_t)$, $\phi(x_{t+k})$, and $k$, instead to predict $a_t$ given $\phi(x_t)$, $\phi(x_{t+k})$, and $k'$, where $k'$ is an \textit{upper bound} on the true value of $k$ (i.e., $k\leq k'$). Note that this formulation is similar to saying that we are given  $\phi(x_t)$, $\phi(x_{t+k'})$, and $k'$, but this path of length $k'$ may contain some number ($k' -k$) of artificially-inserted self-edges. (The only difference is that this second formulation does not allow the \textit{first} action, the predicted action $a_t$, to be the self-edge, but this distinction is a minor one.) This is an appealing picture, because it directly addresses the core problem with the AC-State method, which is that in order to distinguish two states, they must be \textit{exactly} the same distance $k$ (on some path) from a third state. By allowing some imprecision in $k$, this method would seem to address this issue.
\subsubsection{Why it fails}
The problem with this approach is that duplicating an observation effectively ``pauses'' the \textit{exogenous} state of the Ex-BMDP, not just the endogenous state. This modification makes the exogenous state seem to be ``controllable'', and therefore may cause the encoder to ``leak'' information about the exogenous state. Concretely, an incorrect encoder $\phi'$, with a \textit{larger} range of output ``states'' than a correct minimal-range encoder $\phi$, will have a lower loss than $\phi$.

To give an example, consider the Ex-BMDP defined by control-endogenous states $\mathcal{S} = \{a,b,c\}$, exogenous states $\mathcal{E} = \{0,1\}$, actions $\mathcal{A} = \{L,R\}$ and transitions shown in Figure \ref{fig:self_transitions_dont_help}.
\begin{figure}[h!]
\centering
\begin{subfigure}{0.45\textwidth}
\centering
\begin{tikzpicture}
\node[draw=black,shape=circle](a){a};
\node[draw=black,shape=circle](b)[below right = of a]{b};
\node[draw=black,shape=circle](c)[below left = of a]{c};
\path  (a) edge [bend left]["R",swap](b);
\path  (a) edge [bend left]["L"](c);
\path  (b) edge [bend left]["L/R"](c);
\path  (c) edge [bend left]["L/R"](a);
\end{tikzpicture}
    \caption{Endogenous transitions $T(s,a)$}
\end{subfigure}
\centering
\begin{subfigure}{0.45\textwidth}
\centering
\begin{tikzpicture}
\node[draw=black,shape=circle](0){0};
\node[draw=black,shape=circle](1)[right = of 0]{1};
\path  (1) edge [bend right]["p = 1.0",swap](0);
\path  (0) edge [bend right]["p = 1.0",swap](1);
\end{tikzpicture}
    \caption{Exogenous transitions $\mathcal{T}_e(e|e_t)$}
\end{subfigure}
    \caption{An Ex-BMDP with $|\mathcal{S}| = 2$} \label{fig:self_transitions_dont_help}
\end{figure}
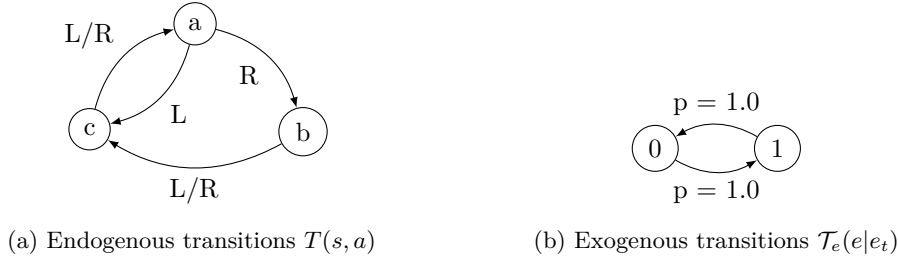

The observations $x \in \mathcal{X}$ are defined by simply concatenating the endogenous and exogenous state labels: $\mathcal{X} = \{a0,a1,b0,b1,c0,c1\}$. Let $\phi$ be a ``correct'' minimal-state encoder, which maps $a0$ and $a1$ to $a$, and so on.

The diameter of these dynamics is $D=2$, and this value is indeed sufficient to learn all endogenous state representations (although the maximum witness distance is in fact $D'=3$).  

Let's consider what happens if we use ``imprecise $k$'' multi-step inverse, with $k' = 2$. Suppose we have that $x_t = a0$, $x_{t+k} = c1$. If we use the encoder $\phi$, our inverse model is only given that $s_t = a$, $s_{t+k} = c$, and $k \in \{1,2\}$. Then $a_t$ can either be L or R, because one can reach $c$ from $a$ in one step with action L, or in two steps with action R. The exact value of $\Pr(a_t=L|s_t = a, s_{t+k} = c, k\in \{1,2\})$ will depend on specifics of the implementation: either the distribution from which $k$ is drawn given $k'$, or equivalently, the probability of taking an artificial self-edge (that is, of duplicating any given state on the replay buffer). However, crucially,  $\Pr(a_t=L|s_t = a, s_{t+k} = c, k\in \{1,2\})$\textit{ will not equal 1}.
However, if we instead use a \textit{less minimal} encoder $\phi'$, which distinguishes a0 from a1 and c0 from c1, we can infer by parity  that $k$ is in fact equal to 1, from the facts that $x_t = a0$,  $x_{t+k} = c1$, and $k \in \{1,2\}$. We then know that $\Pr(a_t=L) =1$, because taking $L$ is the only way to reach $c$ from $a$ in a single step. Therefore, the inverse dynamics model will have a smaller loss if the learned encoder $\phi'$ outputs extra exogenous information. Then this loss function is \textit{not} minimized by the minimal endogenous state encoder.

\section{Numerical Simulation Experiment Details} \label{sec:numerical_apdx}

\subsection{Method}

For each Ex-BMDP shown in Figure \ref{fig:numeric_results}, we perform 50 data-collection runs, and use this data for each value of $K$ and each loss function, to generate 50 trial encoders each. We then determine whether each resulting encoder is either (a) a minimal control-endogenous latent representation, (b) a non-minimal but still correct control-endogenous latent representation, or (c) an incorrect encoder.
For each trial run, we:
\begin{itemize}
    \item Collect two trajectories, each of $T$ timesteps, used as ``training'' and ``validation'' datasets. (This setup is similar to the setup of \cite{lamb2022guaranteed}, where a single trajectory is used for training.)
    \item Iterate over all possible encoders $\phi$. For each $\phi$, we:
        \begin{itemize}
            \item Fit the classifiers $f$ and ($g$ if applicable) on the ``training'' trajectory, with:
            \begin{equation*}
                f_a(s,s',k) := \frac{\text{Freq. of } (s, a) \xrightarrow[]{(k)} s' \text{ in Train.}}{\text{Freq. of } s \xrightarrow[]{(k)} s' \text{ in Train.} }
            \end{equation*}
            \begin{equation*}
                g_{s'}(s,a) := \frac{\text{Freq. of } (s, a) \xrightarrow[]{(1)} s' \text{ in Train.}}{\text{Freq. of } (s,a) \text{ in Train.} }
            \end{equation*}
        
        Note that the states $s,s'$ here are the \textit{encoded} states under $\phi$. For both classifiers, if the denominator is zero, we set the distribution as uniform over actions/latent states. If the numerator is zero but the denominator is not, we set the probability as $10^{-7}$ to avoid infinite losses.
    \item Evaluate the expected loss (Equation \ref{eq:multistep_inverse} or \ref{eq:acdf}) on the ``validation'' trajectory, and record the expected loss as $\mathcal{L}(\phi)$.  (The train/validation split is necessary to avoid over-fitting.)
    \end{itemize}
    \item Find the $\phi$ with the minimum loss. Among the $\phi$'s with loss within 0.1\% of this overall minimum loss, choose the one with the smallest number of latent states to return. 
\end{itemize}
Note that the number of possible encoders grows extremely quickly with $|\mathcal{X}|$ ($> 10^5$ for 10 states), so we limit our examples to cases with  $|\mathcal{X}|=10$. We also use a uniform random behavioral policy, rather than incorporating planning for exploration as in \cite{lamb2022guaranteed}, in order to avoid repeatedly searching for the optimal encoder.

One minor technical caveat is that, for the sake of efficient parallel computation, when measuring the frequency of a transition $(s, a) \xrightarrow[]{(k)} s'$, we only consider spans $s_{t} \rightarrow s_{t+k}$ for $t \in [0,T-K_\text{max}-1]$, where $K_\text{max}$ is the largest $K$ considered in the experiment (e.g., $K_\text{max} = 7$ for the first experiment in Figure \ref{fig:numeric_results}). This wastes a small number (up to  $K_\text{max} -k$) of possible samples from the trajectory. However, this is negligible compared to the overall length of the trajectory $T$, and in any case we would not expect this to bias us towards either AC-State or ACDF. (This approximation was also used when computing one-step frequencies for ACDF.)

\subsection{Additional Results}
\subsubsection{Success Rates for Correct \textit{Minimal} Encoders}
In the previous section, we noted that a returned encoder can either be (a) a minimal control-endogenous latent representation, (b) a non-minimal but still correct control-endogenous latent representation, or (c) an incorrect encoder. In the results shown in the main text, we consider a ``success'' as either case (a) \textit{or} case (b). That is, we consider all correct encoders as successes, even if they are not minimal-state. Here, in Figure \ref{fig:numeric_results_minimal}, we present the results considering only minimal-state encoders as successful.  Interestingly, this only differed at all from the results shown in the main text in two cases (the top and bottom row examples of Figure \ref{fig:numeric_results}), so we only show results for these cases. For the other two Ex-BMDPs, neither method ever returned a correct but non-minimal encoder.

\begin{figure}[h]
    \centering
    \includegraphics[width=\textwidth]{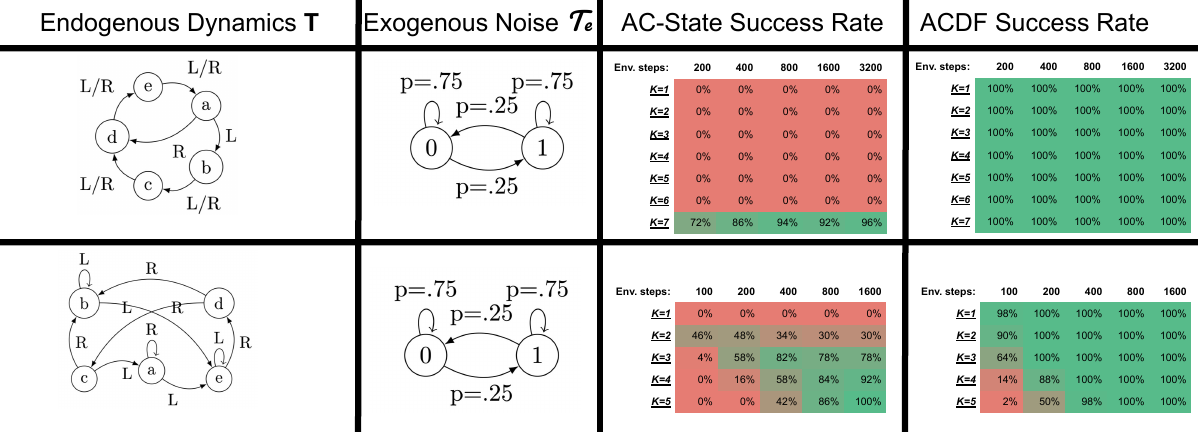}
    \caption{Results of numerical simulation experiments, where the ``Success Rate'' includes only cases where the learned encoder is both correct and state-minimal. For the other two examples in Figure \ref{fig:numeric_results}, no correct but non-minimal encoders were returned, so the results are identical to those shown in the main text.}
    \label{fig:numeric_results_minimal}
\end{figure}
\subsubsection{Complete Results for the Deterministic Ex-BMDP Example (Second Row of Figure \ref{fig:numeric_results})}
In the example shown on the second row of Figure \ref{fig:numeric_results}, we tested over a larger range of $K$ than could fit in the figure. The complete results are shown in Figure \ref{fig:numeric_results_b}.
\begin{figure}
    \centering
\includegraphics[width=\textwidth]{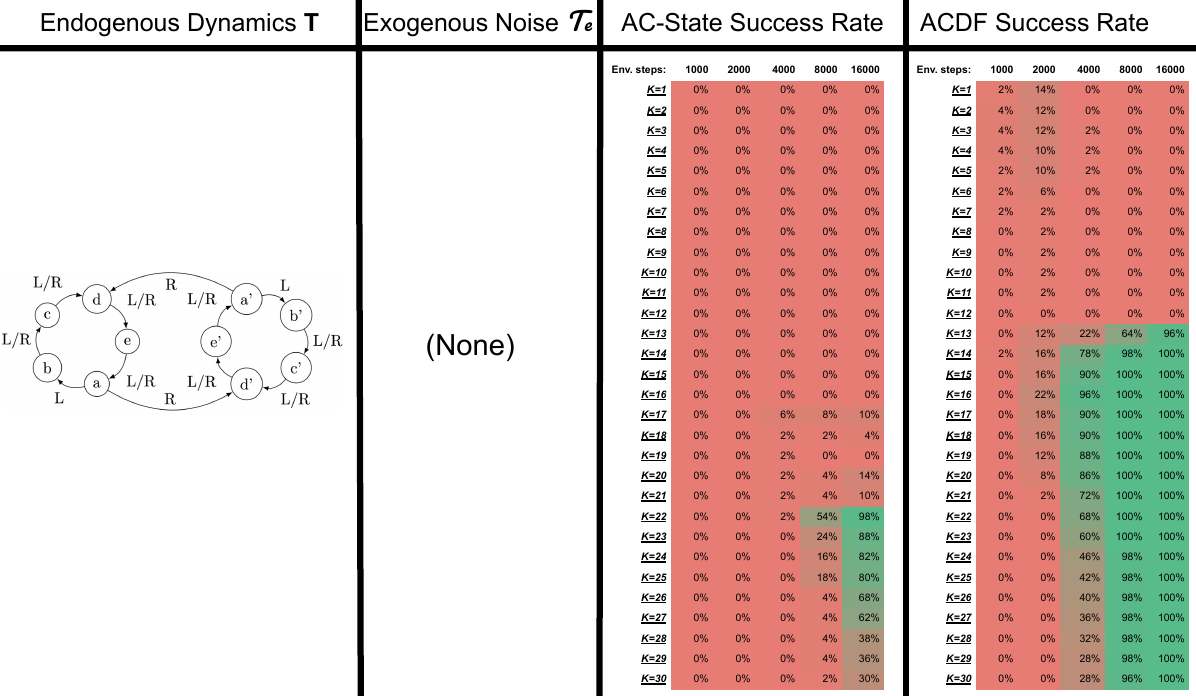}
    \caption{Full results for the numerical simulation experiment shown in the second row of Figure~\ref{fig:numeric_results}.}
    \label{fig:numeric_results_b}
\end{figure}

\section{Deep RL Experiment Details} \label{sec:DRL}
In this section, we describe the deep representation learning experiments we conducted, where the final results are shown in Table \ref{tab:drl_results}. 
\subsection{Environments}
We tested on two environments, both shown in Figure \ref{fig:drl_envs}. The first ``baseline'' environment is taken from the released code from 
\cite{lamb2022guaranteed}, and is similar to the environment described in Section 6.2 of that paper (specifically, it is similar the variant of the environment from that section without ``reset'' actions). The baseline environment consists of nine copies of a four-room maze, where the controllable agent navigates using four actions (up/down/left/right), in just one of the nine mazes. The other eight mazes have other ``agents'' in them which take random actions. The observation is a 11~$\times$~99  image of all nine mazes. For this environment, we start at a random configuration and run for a single trajectory of 5000 steps. 

The second environment we tested is designed to be similar in format to the baseline environment, but to specifically highlight the flaws of AC-State. This environment has a periodic transition function, and an action space consisting of two possible actions. To accomplish the periodicity, in this environment, each of the nine identical mazes consists of a track of 40 states, where in most states, the agent simply moves to the next state regardless of action.  However, in every fifth state, there is an action which transports the agent to some other position in the track, in the pattern shown in Figure \ref{fig:drl_envs}. In each of these long-distance jumps, the agent either moves back or forward by a multiple of 10 states. This gives the overall dynamics a periodicity of 10. The maximum witness distance $D'$ of these dynamics is also 10. For data collection, we again collect 5000 transitions; here, we do so as 25 trajectories of length 200 each. In each trajectory, we initialize the state at a random configuration. (The reason for using multiple trajectories is that, because both the ego-agent and the distractor agents have periodic dynamics with the same periodicity, a single episode will not cover the full configuration space.)

\subsection{Architecture and Hyperparameters}
In general, unless otherwise noted here, we use exactly the architecture and training hyperparameter settings that are default in the released code of \cite{lamb2022guaranteed} for gridworld exploration environments.\footnote{In particular, we used exactly the network architecture for the encoder provided by \cite{lamb2022guaranteed}'s released code, for both AC-State and ACDF. We have observed that this  architecture design appears to be specifically well-suited to easily learn to filter out the exogenous ``distractor'' mazes in these multi-maze environments, and may not be more generally applicable. However because we use the same architecture for both ACDF and AC-State, our direct comparison of these methods is valid.}

\subsection{Training Schedule and Behavior Policy}
We conducted our experiments in an ``offline'' setting: we collected all 5000 transitions for each each environment under uniformly random actions. We then performed 30,000 training iterations, evaluating after every 5000 steps. When evaluating, we evaluated using the version of $\phi_\theta$ from the previous training steps which achieved the lowest training loss using a rolling average over 20 batches. We exclude the first 3000 training iterations when selecting the lowest-loss previous $\phi_\theta$.

\subsection{Implementation of the forward-dynamics loss}
To implement the forward dynamics loss of ACDF (Equation \ref{eq:acdf}), for the dynamics model $g$ we use a LeakyReLU MLP consisting of four layers of sizes $[512+10,1024,1024,N]$ where the input is the discretized vector output of the vector-quantization output layer of the state encoder, representing the encoded state (a vector of size 512), concatenated with the action label. The output is the discrete \textit{index} of the vector-quantized latent representation code of the next latent state. (The number of codes is a hyperparameter of AC-State which we vary.) Note that because we treat the next-state as a discrete index, we do not backpropagate into the encoder of $\phi_\theta(x_{t+1})$, only into the encoder of $\phi_\theta(x_{t})$. This is to mitigate the ``moving-target'' issue mentioned at the beginning of Appendix \ref{sec:alt_fixes}. Also to mitigate this issue, we \textit{only} use the forward prediction loss to update the encoder $\phi_\theta$ every fifth training iteration: at other iterations, we update the parameters of $g$ alone (and separately update $\phi_\theta$ with the multistep-inverse loss). 
\begin{figure}[t]
    \centering
    \includegraphics[width=\textwidth]{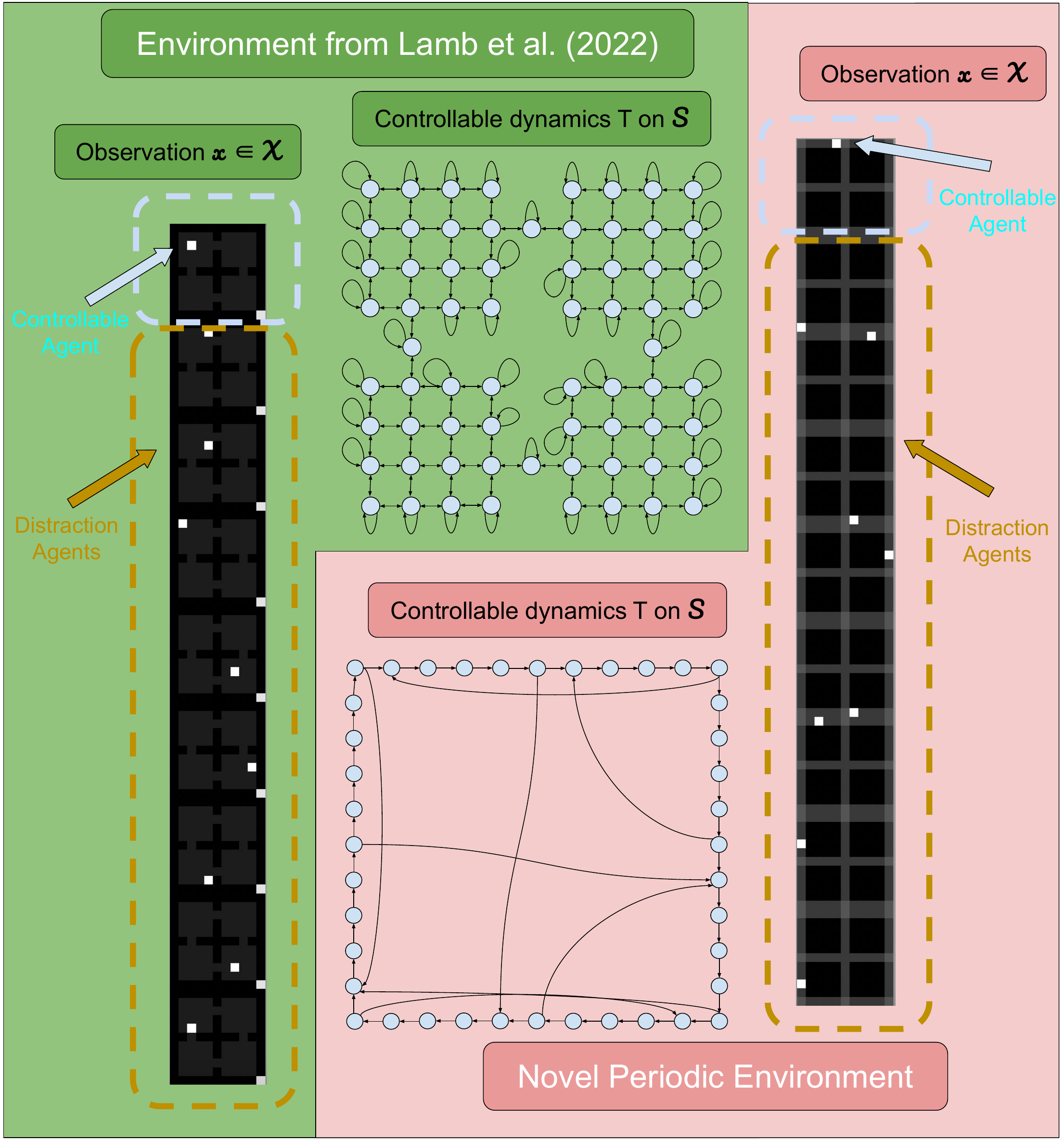}
    \caption{The two high-dimensional environments we tested on. For both environments, we show the image observations $x\in \mathcal{X}$ as well as the  states $S^*$ and deterministic state transitions $T^*$  in the controllable latent dynamics (for readability, we do not show action labels). In these environments, the controllable latent dynamics only represent the top maze in the image: the other mazes contain agents which move randomly.}
    \label{fig:drl_envs}
\end{figure}

\subsection{Evaluation}
In order to capture a real-world usability of the learned representations in  planning, we measure success based on the ability to perform open-loop planning to successfully reach a goal.
\begin{itemize}
    \item We first infer a count-based tabular representation of the deterministic forward dynamics $T$ on the encoded states output by $\phi_{\theta}$, by passing  each observation of all transitions in the dataset through the final learned $\phi_{\theta}$, and inferring that the correct latent transition for a given latent state and action is the one which occurs most often. 
    \item Using this graph, we repeat the following test 1000 times:
\begin{itemize}
    \item We first select two random observations $x$ and $x'$, representing a start state and an end state, from the observation space $\mathcal{X}$ if the environment.
    \item We record the ground-truth controllable-agent state for both observations, and compute the learned latent states $\phi_\theta(x)$ and $\phi_\theta(x')$. We then use our count-based tabular representation of the learned latent transition function to plan (using Dijkstra's algorithm) a shortest path from $\phi_\theta(x)$ to $\phi_\theta(x')$.  This generates a sequence of actions that should (ideally) navigate from $x$ to $x'$.
    \item We execute this sequence of actions from $x$ in an open-loop manner.
    \item We count the trial as a success if the ground-truth controllable-agent state after executing these actions is equal to the ground-truth controllable-agent state of $x'$.
\end{itemize}
\item The overall success rate over the 1000 trials is recorded as a percentage. If it is at least 98\% (980/1000), the representation is considered  to be successfully learned.
\end{itemize}

\subsection{Results}
For both environments and both methods, we ran a hyperparameter sweep over $K$, the number of steps used in multistep-inverse, and $N$, the maximum number of latent states (i.e., the size of the codebook of the final vector-quantization layer of the encoder). For each hyperparameter setting, we ran on 10 seeds. We then evaluated on the ``best'' version of the hyperparameters for each method for 20 additional seeds. Our final results reflect the rate of successful representation learning at the 98\% planning success threshold, at training step 30,000, as described above. Note that the final averages are over the 20 additional seeds alone, to avoid multiple-comparison issues. 

To select the ``best'' hyperparameter setting, we used the following criteria:
\begin{itemize}
    \item We first used the percentage of successfully-learned representations out of the 10 seeds after 30,000 iterations, based on the 98\%-planning success threshold for each seed.
    \item To break ties, we used the number successfully-learned representations over all 10 seeds and all 6 evaluation times (every 5000 training steps) for each seed. This rewards configurations which learn a successful encoder early, as well as those which may have arrived on successful encoders by chance but then later found a lower-loss encoder for the loss function which was incorrect.
    \item To break further ties, we used the percentage of \textit{all} open-loop planning trials, over all evaluation times and all seeds, on which the agent succeeded. Note, however, that this is often a misleading statistic, because an incorrect representation of the environment's dynamics might still happen to wind up on the correct place with considerable frequency.
\end{itemize}

We report the full hyperparameter-sweep results over each of these statistics here (Tables \ref{tab:res_sbf}-\ref{tab:res_cpp}). We note that while both methods are able to consistently learn the controllable dynamics on the baseline environment, only ACDF consistently learns the dynamics of the periodic environment. By contrast, there is only a single hyperparameter configuration (K=1, N=120) where the lowest-loss encoder under AC-State at the final training iteration was correct, and this was \textit{for only two out of ten  random seeds}. (This result is robust to changes in the threshold planning accuracy to be considered a correct representation: these two seeds were the only runs out of the hyperparameter sweep where AC-State succeeded at the final training iteration for success thresholds as low as 75\%.)

In the final results shown in the main text, we show that both methods succeeded on the baseline environment on all 20 seeds, while ACDF succeeded on the periodic environment in 19/20 seeds\footnote{For the remaining one seed, the open-loop navigation task succeeded in 97.5\% of trials, falling just barely under the 98\% threshold we set to consider the run a success.}, and AC-State only succeeded with one training seed out of 20 on this environment. This is a highly statistically significant difference. Using the Clopper-Pearson method, the 99\% CI for the ACDF success rate on this environment is $(.68-0.9997)$, while the 99\% CI for AC-State is $(.0003-.32)$. The final open-loop planning accuracies for each of the 20 evaluation seeds for each method and each environment are reported in Tables \ref{tab:final_results_baseline} and \ref{tab:final_results_periodic}.

\begin{table}[h]
    \centering
    \begin{tabular}{|c|c|c|c|c|c|}
    \hline
&N=68&N=78&N=88&N=98&N=108\\
\hline
K=3&10.0\%&80.0\%&90.0\%&100.0\%&100.0\%\\
K=5&10.0\%&80.0\%&90.0\%&100.0\%&100.0\%\\
K=7&0.0\%&60.0\%&60.0\%&90.0\%&100.0\%\\
    \hline
    \end{tabular}
    \caption{AC-State, Baseline Environment, Final-Iteration Complete-Representation Success Rate}
    \label{tab:res_sbf}
\end{table}
\begin{table}[h]
    \centering
    \begin{tabular}{|c|c|c|c|c|c|}
    \hline
&N=68&N=78&N=88&N=98&N=108\\
\hline
K=3&3.3\%&58.3\%&61.7\%&78.3\%&80.0\%\\
K=5&3.3\%&61.7\%&71.7\%&76.7\%&73.3\%\\
K=7&0.0\%&45.0\%&46.7\%&68.3\%&80.0\%\\
    \hline
    \end{tabular}
    \caption{AC-State, Baseline Environment, All-Iterations Complete-Representation Success Rate}
    \label{tab:res_sba}
\end{table}
\begin{table}[h]
    \centering
    \begin{tabular}{|c|c|c|c|c|c|}
    \hline
&N=68&N=78&N=88&N=98&N=108\\
\hline
K=3&68.8\%&85.8\%&88.4\%&91.5\%&91.2\%\\
K=5&66.4\%&83.7\%&85.4\%&88.7\%&86.8\%\\
K=7&60.0\%&73.7\%&78.4\%&85.7\%&88.4\%\\
    \hline
    \end{tabular}
    \caption{AC-State, Baseline Environment, All-Iterations Path-Planning Success Rate}
    \label{tab:res_sbp}
\end{table}

\begin{table}[h]
    \centering
    \begin{tabular}{|c|c|c|c|c|c|}
    \hline
&N=68&N=78&N=88&N=98&N=108\\
\hline
K=3&10.0\%&90.0\%&100.0\%&100.0\%&100.0\%\\
K=5&40.0\%&90.0\%&100.0\%&100.0\%&100.0\%\\
K=7&10.0\%&80.0\%&100.0\%&100.0\%&100.0\%\\
    \hline
    \end{tabular}
    \caption{ACDF, Baseline Environment, Final-Iteration Complete-Representation Success Rate}
    \label{tab:res_cbf}
\end{table}
\begin{table}[h]
    \centering
    \begin{tabular}{|c|c|c|c|c|c|}
    \hline
&N=68&N=78&N=88&N=98&N=108\\
\hline
K=3&6.7\%&78.3\%&83.3\%&81.7\%&83.3\%\\
K=5&16.7\%&61.7\%&76.7\%&80.0\%&76.7\%\\
K=7&1.7\%&48.3\%&65.0\%&73.3\%&76.7\%\\
    \hline
    \end{tabular}
    \caption{ACDF, Baseline Environment, All-Iterations Complete-Representation Success Rate}
    \label{tab:res_cba}
\end{table}
\begin{table}[h]
    \centering
    \begin{tabular}{|c|c|c|c|c|c|}
    \hline
&N=68&N=78&N=88&N=98&N=108\\
\hline
K=3&72.1\%&87.5\%&88.0\%&88.3\%&88.7\%\\
K=5&72.3\%&80.4\%&82.6\%&84.4\%&85.0\%\\
K=7&65.0\%&76.1\%&77.9\%&81.3\%&82.4\%\\
    \hline
    \end{tabular}
    \caption{ACDF, Baseline Environment, All-Iterations Path-Planning Success Rate}
    \label{tab:res_cbp}
\end{table}

\begin{table}[h]
    \centering
    \begin{tabular}{|c|c|c|c|}
    \hline
&N=40&N=80&N=120\\
\hline
K=1&0.0\%&0.0\%&20.0\%\\
K=5&0.0\%&0.0\%&0.0\%\\
K=10&0.0\%&0.0\%&0.0\%\\
    \hline
    \end{tabular}
    \caption{AC-State, Periodic Environment, Final-Iteration Complete-Representation Success Rate}
    \label{tab:res_spf}
\end{table}
\begin{table}[h]
    \centering
    \begin{tabular}{|c|c|c|c|}
    \hline
&N=40&N=80&N=120\\
\hline
K=1&0.0\%&0.0\%&11.7\%\\
K=5&0.0\%&0.0\%&1.7\%\\
K=10&0.0\%&0.0\%&0.0\%\\
    \hline
    \end{tabular}
    \caption{AC-State, Periodic Environment, All-Iterations Complete-Representation Success Rate}
    \label{tab:res_spa}
\end{table}
\begin{table}[h]
    \centering
    \begin{tabular}{|c|c|c|c|}
    \hline
&N=40&N=80&N=120\\
\hline
K=1&15.3\%&28.8\%&44.3\%\\
K=5&9.2\%&16.6\%&26.5\%\\
K=10&4.7\%&7.9\%&15.6\%\\
    \hline
    \end{tabular}
    \caption{AC-State, Periodic Environment, All-Iterations Path-Planning Success Rate}
    \label{tab:res_spp}
\end{table}

\begin{table}[h]
    \centering
    \begin{tabular}{|c|c|c|c|}
    \hline
&N=40&N=80&N=120\\
\hline
K=1&10.0\%&100.0\%&100.0\%\\
K=5&0.0\%&0.0\%&0.0\%\\
K=10&0.0\%&0.0\%&0.0\%\\
    \hline
    \end{tabular}
    \caption{ACDF, Periodic Environment, Final-Iteration Complete-Representation Success Rate}
    \label{tab:res_cpf}
\end{table}
\begin{table}[h]
    \centering
    \begin{tabular}{|c|c|c|c|}
    \hline
&N=40&N=80&N=120\\
\hline
K=1&1.7\%&48.3\%&46.7\%\\
K=5&0.0\%&0.0\%&0.0\%\\
K=10&0.0\%&0.0\%&0.0\%\\
    \hline
    \end{tabular}
    \caption{ACDF, Periodic Environment, All-Iterations Complete-Representation Success Rate}
    \label{tab:res_cpa}
\end{table}
\begin{table}[h]
    \centering
    \begin{tabular}{|c|c|c|c|}
    \hline
&N=40&N=80&N=120\\
\hline
K=1&54.8\%&65.5\%&65.9\%\\
K=5&2.1\%&4.2\%&4.4\%\\
K=10&2.3\%&2.0\%&1.9\%\\
    \hline
    \end{tabular}
    \caption{ACDF, Periodic Environment, All-Iterations Path-Planning Success Rate}
    \label{tab:res_cpp}
\end{table}

\begin{table}[]
    \centering
    \begin{tabular}{|c|c|}
             \hline
         AC-State & 100.0\%, 100.0\%, 100.0\%, 100.0\%, 100.0\%, 100.0\%, 100.0\%, 100.0\%, 100.0\%,\\
         K=1, N=108&  100.0\%, 100.0\%, 100.0\%, 100.0\%, 99.9\%, 99.9\%, 99.9\%, 99.8\%, 99.8\%, 99.7\%, 99.5\% \\
         \hline
        ACDF & 100.0\%, 100.0\%, 100.0\%, 100.0\%, 100.0\%, 100.0\%, 100.0\%, 100.0\%, 100.0\%,\\
        K=1, N=108 &100.0\%, 100.0\%, 100.0\%, 100.0\%, 99.9\%, 99.9\%, 99.9\%, 99.9\%, 99.9\%, 99.9\%, 99.8\%\\
                 \hline
    \end{tabular}
    \caption{Open-loop path planning accuracy for all 20 random seeds for AC-State and ACDF on the baseline environment at final evaluation, after hyperparameter optimization. Values are sorted in descending order of accuracy.}
    \label{tab:final_results_baseline}
\end{table}

\begin{table}[]
    \centering
    \begin{tabular}{|c|c|}
             \hline
         AC-State & 99.6\%, 74.6\%, 72.0\%, 65.4\%, 64.5\%, 64.4\%, 60.2\%, 55.0\%, 53.8\%, 36.3\%,\\
         K=1, N=120& 34.8\%, 33.9\%, 30.8\%, 27.8\%, 27.5\%, 24.4\%, 23.8\%, 23.6\%, 19.8\%, 19.0\% \\
         \hline
        ACDF & 100.0\%, 99.9\%, 99.9\%, 99.9\%, 99.8\%, 99.7\%, 99.7\%, 99.7\%, 99.7\%, 99.6\%,\\
        K=1, N=80 & 99.5\%, 99.5\%, 99.4\%, 99.0\%, 99.0\%, 99.0\%, 98.9\%, 98.9\%, 98.6\%, 97.5\%\\
                 \hline
    \end{tabular}
    \caption{Open-loop path planning accuracy for all 20 random seeds for AC-State and ACDF on the periodic environment at final evaluation, after hyperparameter optimization. Values are sorted in descending order of accuracy.}
    \label{tab:final_results_periodic}
\end{table}

\end{document}